\documentclass[12pt, draftcls, journal, onecolumn]{IEEEtran}

\usepackage[utf8]{inputenc} 
\usepackage[T1]{fontenc}    
\usepackage{hyperref}       
\usepackage{url}            
\usepackage{booktabs}       
\usepackage{amsfonts}       
\usepackage{nicefrac}       
\usepackage{microtype}      
\usepackage{amsmath,amssymb}
\usepackage{graphicx}
\usepackage{textcomp}
\usepackage{xcolor}
\usepackage{amsthm,amssymb,bm, verbatim,dsfont,mathtools}
\usepackage[mathcal]{euscript}
\usepackage{color,graphicx}
\usepackage{subcaption}
\usepackage{etoolbox}
\usepackage{tikz}
\usepackage{xr,xspace}
\usepackage{todonotes}
\usepackage{soul}
\usepackage{accents}
\usepackage{comment}
\usepackage{cite}
\usepackage{csquotes}
\usepackage{algorithm}
\usepackage{algorithmic}
\usepackage{adjustbox}
\usepackage{multirow,array}

\newtheorem{theorem}{Theorem}
\newtheorem{lemma}{Lemma}

\newcommand{\identityf}[1]{\mathbf 1_{\{#1\}}}
\newlength{\dhatheight}

\begin{document}

\title{Multiple Testing Framework for Out-of-Distribution Detection}

\author
{Akshayaa Magesh,
        Venugopal V. Veeravalli,
        Anirban Roy, 
        Susmit Jha
\thanks{A. Magesh and V.V. Veeravalli are with the Department of Electrical and Computer Engineering, University of Illinois at Urbana-Champaign, IL, 61820 USA. (email: amagesh2@illinois.edu, vvv@illinois.edu).
A. Roy and S. Jha are with SRI International. (email: anirban.roy@sri.com, susmit.jha@sri.com).}
\thanks{This work was supported in part by the National Science Foundation under grants \#2106727 and \#1740079, and by the Army Research Laboratory under Cooperative Agreement W911NF-17-2-0196, through the University of Illinois at Urbana-Champaign.}
\thanks{Part of this work was presented at the ICML 2022 Workshop on Distribution-Free Uncertainty Quantification.}
}

\maketitle

\begin{abstract}
We study the problem of Out-of-Distribution (OOD) detection, that is, detecting whether a Machine Learning (ML) model's output can be trusted at inference time. While a number of tests for OOD detection have been proposed in prior work, a formal framework for studying this problem is lacking. We propose a definition for the notion of OOD that includes both the input distribution and the ML model, which provides insights for the construction of powerful tests for OOD detection. We also propose a multiple hypothesis testing inspired procedure to systematically combine any number of different statistics from the ML model using conformal p-values. We further provide strong guarantees on the probability of incorrectly classifying an in-distribution sample as OOD. In our experiments, we find that threshold-based tests proposed in prior work perform well in specific settings, but not uniformly well across different OOD instances. In contrast, our proposed method that combines multiple statistics performs uniformly well across different datasets and neural networks architectures. 
\end{abstract}

\section{Introduction}

Given the ubiquitous use of ML models in safety-critical applications such as self-driving and medicine,
there is a need to develop methods to detect whether an ML model's output at inference time can be trusted. This problem is commonly referred to as the Out-of-Distribution (OOD) detection problem. If an output is deemed untrustworthy by an OOD detector, one can abstain from making decisions based on the output, and default to a safe action.  There has been a flurry of works on this problem in recent years. A particular area of focus has been on OOD detection for deep learning models  (\cite{mahala, odin, gram, gradnorm, energy, idecode}). While neural networks generalize quite well to inputs from the same distribution as the training distribution, recent works have shown that they tend to make incorrect predictions with high confidence, even for unrecognizable or irrelevant inputs (see, e.g., \cite{szegedy2013intriguing, nguyen2015deep, hendrycks2016baseline}). 

In many of the prior works on OOD detection,  OOD inputs are considered to be inputs that are not generated from 
the input training distribution (see, e.g., \cite{odin, gram}), which better describes the classical problem of outlier detection. However, in contrast to outlier detection, the goal in OOD detection is to flag untrustworthy outputs from a \emph{given} ML model. Thus, it is essential for the definition of an OOD sample to involve the ML model. One of the contributions of this paper is a formal definition for the notion of OOD that involves both the input distribution and the ML model.

In a line of work in OOD detection, it is assumed that the detector has access (exposure) to OOD examples, which can be used to train an auxiliary classifier, or to tune hyperparameters for the detection model (\cite{mahala, oe, odin, liang2022integrative}). Other works rely on identifying certain patterns observed in the training data distribution, and use these patterns to train the original ML model to help detect OOD examples. 
For instance, in \cite{idecode}, a neural network is trained to leverage in-distribution equivariance propoerties for OOD detection. There is another line of work in which tests are designed based on statistics from generative models trained for OOD detection. For instance, in \cite{bergamin2022model}, statistics from a deep generative model are combined through p-values using the Fisher test.  In this paper, we focus exclusively on developing methods that do not use any OOD samples, and can be applied to \emph{any} pre-trained ML model. 

Prior work has primarily been focused on identifying promising test statistics and corresponding thresholds, sometimes motivated by empirical observations of the values taken by these statistics for certain in-distribution and OOD inputs. 
For instance, in \cite{mahala}, a confidence score is constructed through a weighted sum of Mahalanobis distances across layers, using the class conditional Gaussian distributions of the features of the neural network under Gaussian discriminant analysis. In \cite{odin}, a statistic based on input perturbations and temperature-scaled softmax scores is proposed. In \cite{energy}, a \textit{free energy} score based on the denominator of the temperature-scaled softmax score is proposed. In \cite{gram}, scores are derived from Gram matrices, through the sum of deviations of the Gram matrix values from their respective range
observed over the training data. In \cite{angelopoulos2021learn}, the broad goal is to find all candidate functions from a given collection in an offline manner through multiple testing, such that any one of these candidate functions controls some risk at inference time. This approach is applied in \cite{angelopoulos2021learn} to the problem of OOD detection to select suitable thresholds for a given test statistic to control the false alarm rate. In \cite{gradnorm}, vector norms of gradients from a pre-trained network are used to form a test statistic. 
In \cite{haroush2021statistical}, OOD detection in Convolutional Neural Networks (CNNs) is studied; spatial and channel reduction techniques are employed to produce statistics per layer, and these layer statistics are combined to form a final score using a method motivated by the tests proposed by \cite{simes1986improved} and \cite{fisher1992statistical}. Thus, their proposed algorithm computes a single score using all the intermediate features of the CNN and its corresponding empirical p-value. They provide marginal false alarm guarantees averaged over all possible validation datasets used to compute the empirical p-value. Additionally, the proposed method in \cite{haroush2021statistical} can be applied only to CNNs, and not any general ML model.
To summarize, from prior work, it is unclear which among these scores/statistics is the best for OOD detection, or if there exists such a test statistic that is useful for all possible out-distributions. The latter question was raised in \cite{zhang2021understanding}, where they posit that one can construct an out-distribution for any single score or statistic that results in poor detection performance.

The false alarm probability or type-I error of a test refers to the probability of a single in-distribution sample being misclassified as OOD, and the detection power refers to the probability of correctly identifying an OOD distribution sample. Note that the detection power is also referred to as detection accuracy in prior OOD works. In much of the prior work on OOD detection, the false alarm probability is estimated using empirical evaluations on certain in-distribution datasets. What is lacking in such works is a rigorous theoretical analysis of the probability of false alarm, which can be used to meet pre-specified false alarm constraints. Such false alarm guarantees are crucial for the responsible deployment of OOD methods in practice. Note that it is not possible to give any theoretical guarantees on the detection power of an OOD detection test without prior information about the class of all possible out-distributions, which is typically not available in practice. Therefore, in prior work on OOD detection, the detection powers of candidate OOD methods that meet the same pre-specified false alarm levels are compared empirically. 

In this work, we propose a method inspired by \emph{multiple hypothesis testing}~(\cite{holm,bh, benjamini2001}) to systematically combine multiple test statistics for OOD detection. Our method works for combining any number of statistics with an arbitrary dependence structure, for instance the Mahalanobis distances (\cite{mahala}) and the Gram matrix deviations across layers (\cite{gram}) of a neural network. We should emphasize there is no obvious way to directly combine such disparate statistics with provable guarantees for OOD detection. Detection procedures for multiple hypothesis testing are usually based on combining p-values across hypotheses \cite{holm, bh,benjamini2001}. However, in the problem of OOD detection, the probability measures under both the in-distribution (null) and out-of-distribution (alternate) settings are unknown, and thus the actual p-values cannot be computed. In conformal inference methods (\cite{vovk,icad}) the p-values are replaced with \textit{conformal p-values}, which are estimates computed from the empirical CDF of the test statistics. These conformal p-values are data-dependent, as they are calculated from   in-distribution samples. In the 
%
procedure proposed in this paper, we use conformal p-values and provide rigorous theoretical guarantees on the probability of false alarm, conditioned on the dataset used for computing the conformal p-values.

\subsection*{Contributions}

\begin{enumerate}
    \item  We formally characterize the notion of OOD, using which we provide insights on why it is necessary for OOD tests to involve more than just the new unseen input and the final output of the ML model for OOD detection.
    \item We propose a new approach for OOD detection inspired by multiple testing. Our proposed test allows us to combine, in a systematic way, any number of different test statistics produced from the ML model with arbitrary dependence structures.
    \item We provide strong theoretical guarantees on the probability of false alarm, conditioned on the dataset used for computing the conformal p-values. This is stronger than false alarm guarantees in prior work (e.g.,  \cite{icad, idecode}), where the guarantees are given in terms of an expectation over all possible datasets. 
    \item We perform extensive experiments across different datasets to demonstrate the efficacy of our method. We perform ablation studies to show that combining various statistics using our method produces uniformly good results across various types of OOD examples and Deep Neural Network (DNN) architectures.
    
\end{enumerate}

\section{Problem Statement and OOD Modelling}

Consider a learning problem with $(X,Y) \sim \mathrm{P}_{X,Y}$, where $(X,Y)$ is the input-output pair and $\mathrm{P}_{X,Y}$ is the distribution of the dataset available at training time. Let the dataset available at training time be denoted by $\mathcal{T} = \{(X_1,Y_1), (X_2,Y_2), ..., (X_{n}, Y_{n})\}$, where $n$ is the size of the dataset. Let the ML model be denoted by $f(\mathbf{W},.)$, where $\mathbf{W}$ is the random variable denoting the parameters of the ML model. For instance, $\mathbf{W}$ depicts the weights and biases in a neural network. Let $(X_\mathrm{test}, Y_\mathrm{test})$ be a random variable generated from an unknown distribution, and $(x_\mathrm{test}, y_\mathrm{test})$ be an instance of this random variable seen by the ML model at inference time. Given $\mathcal{T}$ and the ML model, the goal is to detect if this new unseen sample might produce an \emph{untrustworthy} output. This might happen because either the input does not conform to the training data distribution, or if the ML model is unable to capture the true relationship between the input $X_\mathrm{test}$ and the true label $Y_\mathrm{test}$. Whether a new unseen sample is OOD or not depends on both the ML model and the distribution $\mathrm{P}_{X,Y}$.

A precise mathematical definition of the OOD detection problem that captures both the input distribution and the ML model appears to be lacking in prior work. The most common definition is based on testing between the following hypotheses (see, e.g., \cite{odin}):
\begin{align} \label{eq:def_odin}
\begin{split}
    &\mathrm{H}_0 : X_\mathrm{test} \sim \mathrm{P}_X \\
    &\mathrm{H}_1 : X_\mathrm{test} \not\sim \mathrm{P}_X,
\end{split}
\end{align}
%
where $\mathrm{H}_0$ corresponds to `in-distribution' and $\mathrm{H}_1$ corresponds to `out-of-distribution'. However, such a definition does not involve the ML model, and better describes the problem of outlier detection, which is fundamentally different from the problem of OOD detection. 

%


Let $\hat{Y} = f(\mathbf{W},X)$, and  consider the distribution  $\mathrm{P}_{X,\hat{Y}} = \mathrm{P}_X \times \mathrm{P}_{\hat{Y} | X}$ as the joint distribution of the input and the output of the ML model. Using this joint distribution as the \textit{`in-distribution'}, consider the following testing problem:
\begin{align}\label{eq:OOD_formulation}
\begin{split}
    &\mathrm{H}_0 : (X_\mathrm{test},Y_\mathrm{test}) \sim \mathrm{P}_{X,\hat{Y}} \\
    &\mathrm{H}_1 : (X_\mathrm{test},Y_\mathrm{test}) \not\sim \mathrm{P}_{X,\hat{Y}}.
\end{split}
\end{align}

Note that this is a definition of OOD detection that involves both the input distribution and the ML model (through $\mathrm{P}_{\hat{Y} | X} = \mathrm{P}_{f(\mathbf{W}, X = x) | X = x}$). It also captures both the cases where the input is not drawn from $\mathrm{P}_X$, and when the ML model is unable to capture the relationship between the unseen input and its label. 

The hypothesis test in \eqref{eq:OOD_formulation} involves the true label ${Y}_\mathrm{test}$ and the distribution $\mathrm{P}_{X,\hat{Y}}$. Since these quantities are unknown, the model prediction $\hat{Y}_\mathrm{test}$ and the empirical distribution of $(X,\hat{Y})$ based on the training data, respectively, may be used instead. When the ML model performs well during training, i.e., $Y = \hat{Y}$ for almost all training data points, the empirical versions of $\mathrm{P}_{X,\hat{Y}}$ and $\mathrm{P}_{X,Y}$ agree, and we again arrive at 
a definition that does not involve the ML model. Thus, we conclude that it is necessary to use other functions of the input derived from the ML model\footnote{In this paper, we use the term \textit{statistic} or \textit{score} interchangeably to denote these functions of the input derived from the ML model.} in addition to just the final output in constructing test statistics for effective OOD detection. Such a strategy is commonly employed, without theoretical justification, in many OOD detection works, for instance, through the use of intermediate features of a neural network to calculate the Mahalanobis score (\cite{mahala}) and gram matrix score (\cite{gram}), and gradient information to calculate the GradNorm score (\cite{gradnorm}). The discussion above provides a qualitative theoretical justification for these strategies developed in prior works. 

\section{Proposed Framework and Algorithm} \label{sec:proposed_algo}

In this section, we describe our proposed framework formally, and present our algorithm to combine any number of different functions of the input with an arbitrary dependence structure.

In our formulation of OOD detection in \eqref{eq:OOD_formulation}, we posit that, in addition to the input and the output from the ML model, it is necessary to use other functions\footnote{Without loss of generality, we may assume that these functions are scalar-valued.} of the input, which are dependent on the ML model. We refer to these functions as \emph{score functions}, denoted by $s^1(.), \ldots, s^K(.)$.  The outputs of the score functions are scalar-valued \emph{scores} $T^1, T^2, \ldots , T^K$:
\begin{align}\label{eq:intermediate_fns}
\begin{split}
    T^1 &= s^1(X) \\
    \vdots  \\
    T^K &= s^K(X).
\end{split}
\end{align}
The score functions are assumed to be chosen based on prior information in such a way that the scores are likely to take on small values for in-distribution inputs and larger values for OOD inputs. For a new input $X_\mathrm{test}$, let $(T^1_\mathrm{test}, T^2_\mathrm{test}, \ldots , T^K_\mathrm{test})$ be the corresponding scores. Note that one of scores $T^k_\mathrm{test}$  could be based on the final output from the learning model $\hat{Y}_\mathrm{test}$.

\subsection{Motivation for multiple testing framework}
In order to construct an OOD detection test for the new sample $X_\mathrm{test}$ using the scores, the scores would need to be combined in some manner. Since we do not know the dependence structure between the scores, combining them in an ad hoc manner, such as summing them up, cannot be justified and may result in tests with low power (probability of detection) for many OOD distributions.  For instance, consider a simple bivariate Gaussian setting as follows: 
\begin{align}
\begin{split}
    &\mathrm{H}_0 : (T^1, T^2) \sim \mathcal{N}(0,I) \\
    &\mathrm{H}_1 : (T^1, T^2) \not\sim \mathcal{N}(0,I).
\end{split}
\end{align}
Let the statistic $T = T^1 + T^2$, and let $Q$ be the p-value when the observed value of the statistic is $t$. 
Recall that the p-value is given by:
\begin{equation}
    Q = \mathrm{P_{H_{0}}}\left\{ T \geq  t \right\}.
\end{equation}
For given $\alpha >0$, let ${\cal T}_1$ denote the test which rejects $\mathrm{H}_0$ if $Q < \alpha$. For test ${\cal T}_1$, the probability of false alarm, i.e.,
\begin{equation}
    \mathrm{P_{H_0}}(\text{reject } \mathrm{H}_0), 
\end{equation}
can be controlled at $\alpha$, by exploiting the fact that p-values have a uniform distribution under the null hypothesis. However, the detection power of the test under different possible distributions under the alternate hypothesis might be poor. For instance, if $(T^1, T^2) \sim \mathcal{N}({(1, -1)}, I)$ under the alternate hypothesis,  
the statistic $T = T^1 + T^2$ has the same distribution under the null and alternate hypotheses. Thus the detection power of test {${\cal T}_1$} is upper bounded by $ \alpha$. It is possible to find many such joint distributions {for} the alternate hypothesis, {under which} the detection power of test {${\cal T}_1$} is {poor}, i.e., it is {close to} the probability of false alarm.

On the other hand, consider the following split of the above testing problem into two binary hypothesis testing problem corresponding to the statistics $T^1$ and $T^2$:
\begin{align}\label{split_test}
\begin{split}
&\mathrm{H}_{0,1} : T^1 \sim \mathcal{N}(0,1) \;\;\;\;\;\; \mathrm{H}_{1,1} : T^1 \not\sim \mathcal{N}(0,1) \\
&\mathrm{H}_{0,2} : T^2 \sim \mathcal{N}(0,1) \;\;\;\;\;\; \mathrm{H}_{1,2} : T^2 \not\sim \mathcal{N}(0,1).
\end{split}
\end{align}
Let $Q^1$ and $Q^2$ be the p-values corresponding to the two individual tests in \eqref{split_test}, and $Q^{(1)} \leq Q^{(2)}$ be the ordered p-values. 
Let 
\begin{equation}
    m = \max \{i: Q^{(i)} \leq i\alpha /2\}.
\end{equation}
Then, let test {${\cal T}_2$} be defined such that it rejects $\mathrm{H}_0$ if $m \geq 1$.

Similar to test {${\cal T}_1$}, the probability of false alarm of test {${\cal T}_2$} can be controlled at level $\alpha$. On the other hand, we see that the detection power of test {${\cal T}_2$} when $(T^1, T^2) \sim \mathrm{P}_\mu$, where $\mathrm{P}_\mu = \mathcal{N}({(\mu_1, \mu_2)}, I)$, satisfies the following condition: 
\begin{equation}\label{bh_performance}
\mathrm{P_{\mu}}(\text{reject } \mathrm{H}_0) 
    \geq 1 - \min \{1 - \mathrm{\Psi}(\mathrm{\Psi}^{-1}(\alpha/2) - \mu_1), 1 - \mathrm{\Psi}(\mathrm{\Psi}^{-1}(\alpha/2) - \mu_2)\},
\end{equation}
where $\mathrm{\Psi}(.)$ is the complementary cumulative distribution function of a $\mathcal{N}(0,1)$ random variable.

Thus, the detection power satisfies a minimum quality of performance under all distributions {for} the alternate hypothesis. Note that it is also possible for some distributions {for} the alternate hypothesis that test {${\cal T}_1$ has better detection power than test ${\cal T}_2$ }\eqref{bh_performance}. For instance, if $(T^1, T^2) \sim \mathcal{N}({(1, 1)}, I)$, then the detection performance of test {${\cal T}_1$} is better than that of {${\cal T}_2$}. However, if we do not have any prior information on the behaviour of the statistics under the 
alternate hypotheses, combining multiple test statistics in an ad hoc manner {(such as summing them) might not be desirable}. Further, there is no obvious way to combine two completely different set of statistics, say the Mahalanobis scores from different layers of a DNN and the energy score. 

\subsection{Proposed OOD Detection Test}
Motivated by the above discussion, we propose the following multiple testing framework for OOD detection:
\begin{align}\label{multiple_test}
\begin{split}
        &\mathrm{H}_{0,1} : T^1_\mathrm{test} \sim \mathrm{P}^1 \qquad \mathrm{H}_{1,1} : T^1_\mathrm{test}\not\sim \mathrm{P}^1 \\ 
        &\vdots \\
        &\mathrm{H}_{0,K} : T^K_\mathrm{test} \sim \mathrm{P}^K \qquad \mathrm{H}_{1,K} : T^K_\mathrm{test} \not\sim \mathrm{P}^K,
\end{split}
\end{align}
where $\mathrm{P}^1, \ldots, \mathrm{P}^K$ are the distributions of the corresponding scores when $X_\mathrm{test}$ is an in-distribution sample as defined in \eqref{eq:OOD_formulation}. It is clear to see that if the new input $X_\mathrm{test}$ is an in-distribution sample, then all $\mathrm{H}_{0,i}$ are true in \eqref{multiple_test}, and if $X_\mathrm{test}$ is an OOD sample, then one or more of $(\mathrm{H}_{0,1}, \ldots, \mathrm{H}_{0,K})$ are likely to be false.  Thus, we propose a test that declares the instance as OOD, if any of the  $\mathrm{H}_{0,i}$ are rejected. 


We propose an algorithm for OOD detection inspired by the Benjamini-Hochberg (BH) procedure {given in} \cite{benjamini2001} (preliminaries are provided in the Appendix). Most multiple testing techniques, including the BH procedure, involve computing the p-values of the individual tests. {The p-value of a realization $t^i$} of the test statistic $T^i$, $i \in [K]$,  is given by 
\begin{equation}
    q^i = \mathrm{P_{H_{0,i}}}\left\{ T^i \geq  t^i \right\} = 1 - F_{\mathrm{H}_{0,i}}(t^i),
\end{equation}
where $F_{\mathrm{H}_{0,i}}(.)$ is the CDF of  $T^i$. The p-value for $T^i_\mathrm{test}$ is a random variable
\begin{equation} \label{eq:Qi}
    Q^i = 1 - F_{\mathrm{H}_{0,i}}(T^i_\mathrm{test}).
\end{equation}
The distribution of this p-value under null hypothesis is uniform over $[0,1]$. Its distribution under the alternate hypothesis concentrates around 0, and is difficult to characterize {in general}. Also, while a p-value close to 0 is evidence against the null hypothesis, a large p-value does not provide evidence in favor of the null hypothesis.

If we do not know the distributions under the null hypotheses to calculate the exact p-values, conformal inference methods suggest evaluating the empirical CDF of $T^i$ under the null hypothesis using a hold-out set (denoted by $\mathcal{T}_\mathrm{cal}$) known as the calibration set, to construct a \textit{conformal p-value} $\hat{Q}^i$. A conformal p-value satisfies the following property: 
\begin{equation}\label{eq:uniform_p_value}
    \mathrm{P_{H_{0,i}}} \left\{ \hat{Q}^i \leq t \right\} \leq t,
\end{equation}
when $X_\mathrm{test}$ is independent from $\mathcal{T}_\mathrm{cal}$ and $T^i$ has a continuous distribution. The classical conformal p-value {(see \cite{vovk})} is given by:
\begin{equation}
    \hat{Q}^ i = \frac{1 + |\{j \in \mathcal{T}_\mathrm{cal}: T^i_j \geq T^i_\mathrm{test}\}|}{1 + |\mathcal{T}_\mathrm{cal}|}.
\end{equation}
The estimate $\hat{Q}^i$ is said to be a marginally valid conformal p-value, as it depends on $\mathcal{T}_\mathrm{cal}$. In other words, \eqref{eq:uniform_p_value} can be rewritten as follows:
\begin{equation}\label{eq:uniform_p_value_expectation}
\mathrm{E}\left[ \mathrm{P_{H_{0,i}}} \left\{ \hat{Q}^i \leq t  |  \mathcal{T}_\mathrm{cal} \right\} \right] \leq t, 
\end{equation}
where the expectation is over all possible calibration datasets. The property in \eqref{eq:uniform_p_value} is however not valid conditionally, i.e., $\mathrm{P_{H_0}} \left\{ \hat{Q}^i \leq t  |  \mathcal{T}_\mathrm{cal} \right\}$ need not be {upper-bounded by} $t$. This is important to note, as false alarm guarantees given for out-of-distribution detection methods using conformal inference (see, e.g., \cite{icad, idecode}) are based on \eqref{eq:uniform_p_value_expectation}. Such guarantees are not strong, as they only guarantee that the probability of false alarm, averaged over all possible calibration data sets, is controlled. While the problem of conditional coverage has been discussed in the context of sequential testing for distribution shifts (e.g., \cite{podkopaev2021tracking}) {and} conformal inference (e.g., \cite{vovk2012conditional}), it has not been discussed widely under the setting of single sample OOD detection.

The related problem of outlier testing using conformal p-values is studied in \cite{bates}. However, the result from \cite{bates}, stating that conformal p-values satisfy the PRDS (Positive Regression Dependent on a Subset) property, which is required for the False Discovery Rate (FDR) control in the BH procedure,  is valid only under the setting where the individual test statistics (and hence the original p-values) are independent. 
The PRDS property does not hold for the conformal p-values $\hat{Q}^i$ in Algorithm 1, since the corresponding p-values $Q^i$ (see \eqref{eq:Qi}) are highly dependent through the common input $X_\mathrm{test}$. In addition, the conditional false alarm guarantees provided in \cite{bates} utilize calibration conditionally valid (CCV) p-values proposed in \cite{bates}, as opposed to the conformal values proposed in \cite{vovk} (which we use in our work). Indeed, these CCV p-values cannot be directly used in our setting to obtain false alarm guarantees in Theorem \ref{main_theorem}, without a similar adjustment to the thresholds as in \eqref{eq:C(K)}, as the p-values would be dependent through both the calibration dataset and the input. 



In our proposed OOD detection test we use conformal p-values in place of the actual p-values. In order to compute the conformal p-values, we maintain a calibration set $\mathcal{T}_\mathrm{cal}$. 

In this work, we aim to provide conditional false alarm guarantees, i.e., if $X_\mathrm{test}$ is an in-distribution sample (all $\mathrm{H}_{0,i}$ are true in \eqref{multiple_test}), then 
\begin{equation}\label{eq:define_P_F}
    \mathrm{P_F}(\mathcal{T}_\mathrm{cal}) = \mathrm{P_{H_0}}(\text{declare OOD } | \mathcal{T}_\mathrm{cal}) = \mathrm{P_{H_0}}(\text{reject at least one } \mathrm{H}_{0,i} | \mathcal{T}_\mathrm{cal})
\end{equation}
is controlled with high probability. As discussed earlier in this section, such conditional guarantees are essential for the safe deployment of OOD detection algorithms. 
Note that in the literature on multiple testing, the marginal false alarm probability $ \mathrm{P_{H_0}}(\text{declare OOD })$ is equivalent to the Family Wise Error Rate (FWER) or False Discovery Rate (FDR) when all the null hypotheses are true in \eqref{multiple_test} (detailed discussion provided in the Appendix).

We compute the scores of these $K$ statistics for the samples in the calibration set $\mathcal{T}_\mathrm{cal}$. Using these, we calculate the conformal p-values $\hat{Q}^1, \hat{Q}^2, \ldots, \hat{Q}^K$ for the new sample as in \eqref{eq:conformal_p_value}, and
order the conformal p-values in increasing order as $\hat{Q}^{(1)}, \hat{Q}^{(2)}, \ldots, \hat{Q}^{(K)}$. Let $\epsilon > 0$ be a parameter of the OOD detection algorithm, and let $\alpha > 0$, 
%
and let 
\begin{equation}\label{eq:m}
	m = \max \left\{i : \hat{Q}^{(i)} \leq \frac{\alpha i}{ C(K) K} \right\},
\end{equation}
where 
\begin{equation}\label{eq:C(K)}
    C(K) = (1+\epsilon)\sum_{j=1}^K \frac{1}{j}.
\end{equation}
The factor of $\sum_{j=1}^K \frac{1}{j}$ is {included in order} to obtain false alarm guarantees for any arbitrary dependence between the test statistics. The factor of $(1+\epsilon)$ is a constant related to the size of the calibration dataset, and is introduced to provide strong conditional false alarm guarantees, conditioned on the calibration set (discussed further in the proof of the results below). While choosing a smaller value of $\epsilon$ improves the power of the proposed OOD detection test, it increases the size of the calibration set needed to provide the conditional false alarm guarantees.
The OOD detection test declares the instance $X_\mathrm{test}$ as OOD if $m \geq 1$, i.e., if any of the $\mathrm{H}_{0,i}$ are rejected. The pseudo-code is described in Algorithm \ref{main}.

\begin{algorithm}[htb]
   \caption{BH based OOD detection test with conformal p-values}
   \label{main}
\begin{algorithmic}
   \STATE {\bfseries Inputs:} 
   \STATE New input $X_\mathrm{test}$; \\ Scores over $\mathcal{T}_{cal}$ as $\left\{ \{T^1_j = s^1(X_j) : j \in \mathcal{T}_\mathrm{cal}\}, \ldots, \{T^K_j = s^K(X_j) : j \in \mathcal{T}_\mathrm{cal}\} \right\}$; \\ ML model $f(\mathbf{W}, .)$; \\ Desired conditional probability of false alarm $\alpha \in (0,1)$.
   \STATE {\bfseries Algorithm:}
   \STATE For $X_\mathrm{test}$, compute scores $T^i_\mathrm{test}$.
   \STATE Calculate conformal p-values as:
   \begin{equation}
       \hat{Q}^ i = \frac{1 + |\{j \in \mathcal{T}_\mathrm{cal}: T^i_j \geq T^i_\mathrm{test}\}|}{1 + |\mathcal{T}_\mathrm{cal}|}.
   \end{equation}
   \STATE Order them as $\hat{Q}^{(1)} \leq  \hat{Q}^{(2)} \leq  \ldots  \leq \hat{Q}^{(K)}$.
   \STATE Calculate $m = \max \left\{i : \hat{Q}^{(i)} \leq \frac{\alpha i}{C(K) K} \right\}$.
   \STATE {\bfseries{Output:}} 
   \STATE Declare OOD if $m \geq 1$.
\end{algorithmic}
\end{algorithm}

For instance, consider a Deep Neural Network (DNN) with $L$ layers. Let $T^1, \ldots, T^L$ denote the Mahalanobis scores (\cite{mahala}). Let $T^{L+1}, \ldots, T^{2L}$ denote the Gram deviation scores (\cite{gram}). \cite{mahala} use outlier exposure to combine $T^1, \ldots, T^L$ into a single score for a threshold-based test, and the \cite{gram} use the sum of $T^{L+1}, \ldots, T^{2L}$ for a similar test. However, it is not straightforward to {determine} how to combine the $L$ Mahalanobis scores and the $L$ Gram deviation scores for OOD detection without outlier exposure. In Algorithm \ref{main}, we {provide} a systematic way to construct a test that uses all these $2L$ contrasting scores. In addition, we provide a systematic way to design the test thresholds to meet a given false alarm constraint as presented below in Theorem \ref{main_theorem}.

\subsection{Theoretical Guarantees}
On running Algorithm \ref{main}, we can guarantee that the conditional probability of the false alarm is bounded by $\alpha$ with high probability. In order to provide this guarantee, we need to enforce certain sample complexity conditions on the size of the calibration set $n_\mathrm{cal}$, as detailed in the Lemma below. 

\begin{lemma}\label{cal_assumption}
Let $\epsilon > 0$, $K$ and $\alpha$ be as in Algorithm 1. Let $a_j = \lfloor (n_\mathrm{cal} + 1) \frac{\alpha j}{C(K) K}  \rfloor$, $b_j = (n_\mathrm{cal}+1) - a_j$, and $\mu_j = \frac{a_j}{a_j + b_j}$. For a given $\delta >0$, let $n_\mathrm{cal}$ be such that
\begin{equation}\label{eq:n_cal}
    \min_{j = 1, 2, \ldots, K} I_{(1+\epsilon)\mu_j} (a_j, b_j) \geq 1 - \frac{\delta}{K^2}, 
\end{equation}
where $I_x(a,b)$ is the regularized incomplete beta function (the CDF of a Beta distribution with parameters $a, b$). Then for random variables $r^i_j \sim \mathrm{Beta} (a_j, b_j)$ for $j = 1, \ldots, K$, 
\begin{equation}
    P \left\{ \bigcap_{i=1}^K \bigcap_{j=1}^K \left\{r_j^i \leq (1 + \epsilon) \frac{\alpha j}{C(K) K} \right\} \right\} \geq 1 - \delta.
\end{equation}
\end{lemma}

\begin{proof}
    When the condition on $n_\mathrm{cal}$ in \eqref{eq:n_cal} is satisfied, we have that 
    \begin{equation}
        \begin{split}
             P \left\{ r_j^i \leq (1 + \epsilon) \frac{\alpha j}{C(K) K} \right\} &= I_{(1+\epsilon) \frac{\alpha j}{C(K) K}} (a_j, b_j) \\
             &\geq I_{(1+\epsilon) \mu_j} (a_j, b_j) \\
             &\geq 1 - \frac{\delta}{K^2},
        \end{split}
    \end{equation}
    where $I_x(a,b)$ is the CDF of a Beta distribution with parameters $a, b$, and the second inequality follows since $\mu_j$ is upper bounded by $\frac{\alpha j}{C(K) K}$. From the Union Bound, we have that, 
    \begin{equation}
        \begin{split}
            1 -  P \left\{ \bigcap_{i=1}^K \bigcap_{j=1}^K \left\{r_j^i \leq (1 + \epsilon) \frac{\alpha j}{C(K) K} \right\} \right\} &\leq \sum_{i=1}^K \sum_{j=1}^K P \left\{ r_j^i \geq (1 + \epsilon) \frac{\alpha j}{C(K) K} \right\} \\
            &\leq \sum_{i=1}^K \sum_{j=1}^K  \frac{\delta}{K^2} \\
            &\leq \delta.
        \end{split}
    \end{equation}

    Thus, we have the desired result in Lemma \ref{cal_assumption}.
\end{proof}

The condition on $n_\mathrm{cal}$ in Lemma \ref{cal_assumption} is due to the fact that the CDF of the conformal p-values conditioned on the calibration dataset follows a Beta distribution (see \cite{vovk}), and is essential to provide the guarantees in Theorem \ref{main_theorem}. Due to the form of the CDF of the Beta distribution, it is difficult to characterize the dependence of $n_\mathrm{cal}$ on $\alpha$, $\delta$, $\epsilon$ and $K$ in closed form. We plot the calibration dataset sizes $n_\mathrm{cal}$ as given by Lemma \ref{cal_assumption} for $\epsilon =1$ and $K = 5$ for different values of $\delta$ in Figure~\ref{fig:sample_complexity}. Note that $\epsilon = 1$ is conservative. 

\begin{figure}[h]
  \centering
  \subcaptionbox{{\small $\alpha$ = 0.1}}[.485\linewidth][c]{%
    \includegraphics[width=1\linewidth]{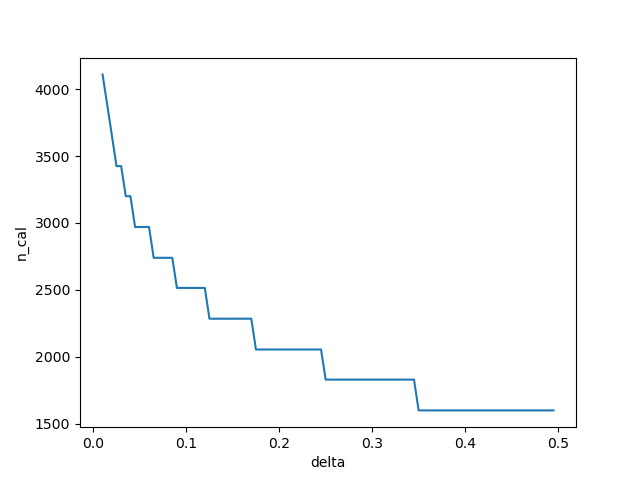}}\quad
  \subcaptionbox{{\small $\alpha$ = 0.05}}[.485\linewidth][c]{%
    \includegraphics[width=1\linewidth]{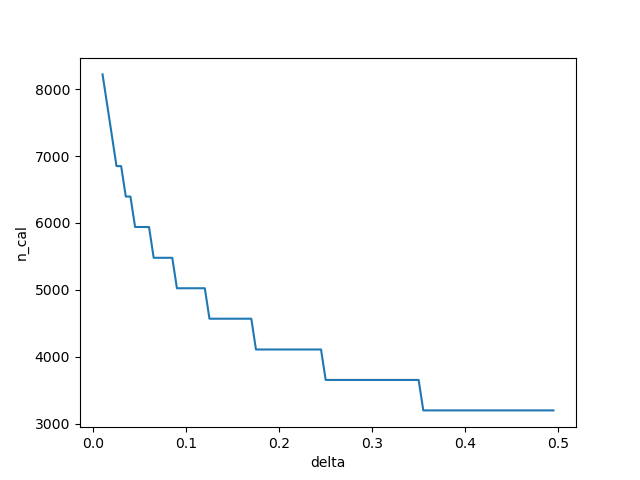}}\quad 
  \caption{Calibration dataset sizes that guarantees Theorem 1 with probability $1 - \delta$}
  \label{fig:sample_complexity}
\end{figure}

%
In the following result, we formally present the conditional false alarm guarantee for Algorithm~\ref{main}.
\begin{theorem}\label{main_theorem}
Let $\alpha, \delta \in (0,1)$. Let $\mathcal{T}_\mathrm{cal}$ be a calibration set,
and let $n_\mathrm{cal}$  be large enough (as defined in the Lemma \ref{cal_assumption}). Then, for a new input $X_\mathrm{test}$ and an ML model $f(\mathbf{W}, .)$, the probability of incorrectly detecting $X_\mathrm{test}$ as OOD conditioned on $\mathcal{T}_\mathrm{cal}$ while using Algorithm \ref{main} is bounded by $\alpha$, \textit{i.e.},
\begin{equation}
    \mathrm{P}_{\mathrm{F}}(\mathcal{T}_\mathrm{cal}) = \mathrm{P_{H_0}}\left({\text{declare OOD }  }  | \mathcal{T}_\mathrm{cal}\right) \leq \alpha, 
\end{equation}
with probability $1 - \delta$.
\end{theorem}
We adapt the proof of FDR control for the BH procedure provided in \cite{benjamini2001} for our algorithm, to the use of conformal p-values estimated from the calibration set instead of the actual p-values in Algorithm \ref{main}. The {details of the proof are} presented in the Appendix.

 \begin{figure}[h]
  \centering
  \subcaptionbox{{\small ResNet with CIFAR10}}[.485\linewidth][c]{%
    \includegraphics[width=1\linewidth]{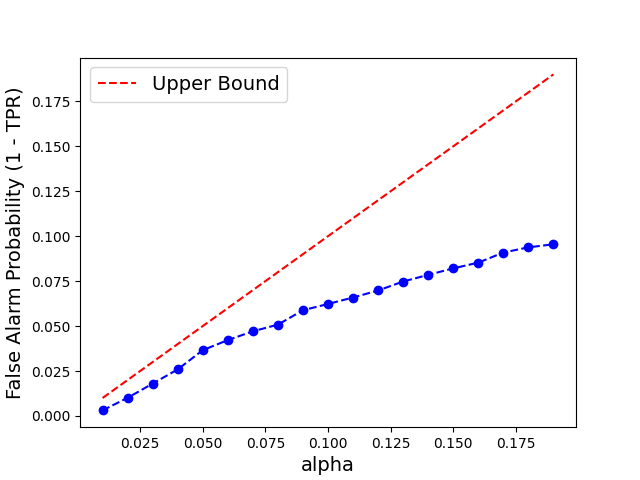}}\quad
  \subcaptionbox{{\small DenseNet with CIFAR10}}[.485\linewidth][c]{%
    \includegraphics[width=1\linewidth]{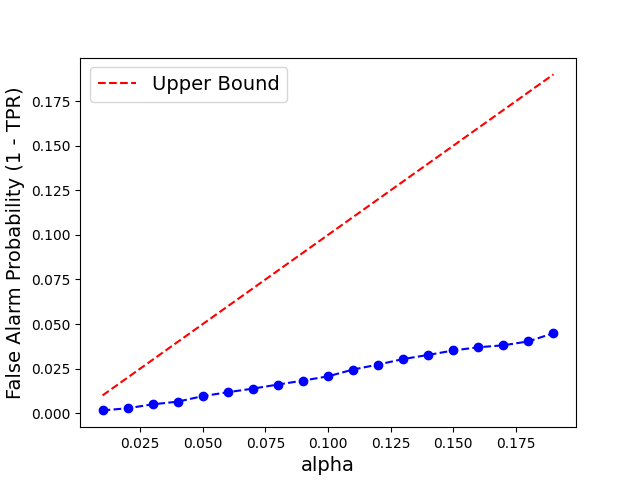}}\quad 
    \bigskip
  \subcaptionbox{{\small ResNet with SVHN}}[.485\linewidth][c]{%
    \includegraphics[width=1\linewidth]{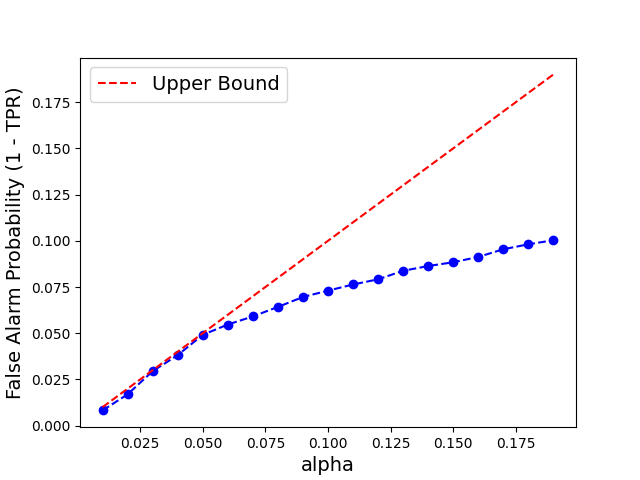}}
  \subcaptionbox{{\small DenseNet with SVHN}}[.485\linewidth][c]{%
    \includegraphics[width=1\linewidth]{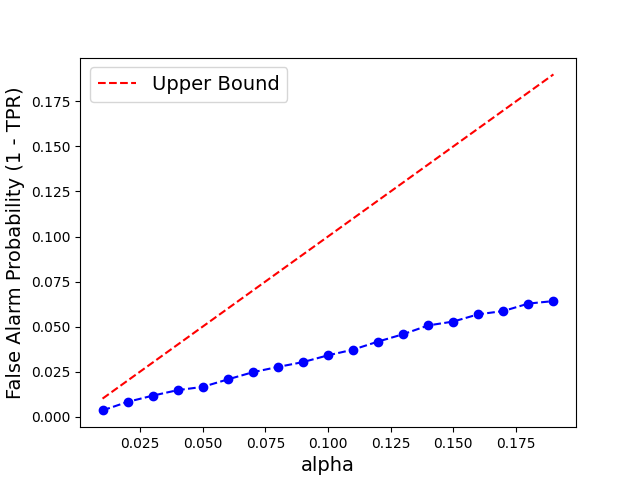}}
  \caption{False Alarm probabilities with CIFAR10 and SVHN as in-distribution datasets for ResNet and DenseNet}
  \label{fig:false_alarm}
\end{figure}

We verify the results in Theorem~\ref{main_theorem} through experiments with CIFAR10 and SVHN as in-distribution datasets, and ResNet and DenseNet architectures (more details on the experimental setup are given in Section \ref{experiments}). In Figure~\ref{fig:false_alarm}, we plot the false alarm probabilities when the thresholds for comparing the conformal p-values are set according to Algorithm~\ref{main}. The dashed line represents the theoretical upper bound on the false alarm probability. As seen in Figure~\ref{fig:false_alarm}, the false alarm probability is bounded by the theoretical upper bound as stated in Theorem~\ref{main_theorem} for all settings considered. Note that the results in this paper hold for any given ML model, and while the bound may be conservative for certain settings (e.g., DenseNet with CIFAR10), it is tight in other cases (e.g., ResNet with SVHN).

Such strong theoretical guarantees are absent in most prior work on OOD detection. A few works that have suggested the use of conformal p-values for OOD detection, such as \cite{idecode}, provide marginal false alarm guarantees of the form:
\begin{equation}
    \mathrm{P_{H_0}}\{\text{declare OOD }\} = \mathrm{E} \left[ \mathrm{P_{H_0}}\{\text{declare OOD } | \mathcal{T}_\mathrm{cal}\} \right] \leq \alpha 
\end{equation}
where the expectation is over all possible calibration sets. (See also the discussion surrounding \eqref{eq:uniform_p_value_expectation}.) However, this does not guarantee that the false alarm level $\alpha$ is maintained with high probability for the particular calibration dataset used. In addition, it does not provide any information on the size of the calibration dataset to be used. 


\section{Experimental Evaluation}\label{experiments}
In the previous section, we have provided guarantees on the strong probability of false alarm for Algorithm \ref{main}. However, since it is not possible to theoretically analyze the power of such a test (due to the structure of the alternate hypothesis), we evaluate the power of our proposed approach through experiments. In addition, since we do not know beforehand what kind of OOD samples might arise at inference time, an effective OOD detection test must also have low variance across different OOD datasets, for a given Deep Neural Network (DNN) architecture. In our experiments, we evaluate both of these metrics to demonstrate the effectiveness of our approach. 

Following the standard protocol for OOD detection~(\cite{mahala, gram, energy}), we consider settings with CIFAR10 and SVHN as the in-distribution datasets.
\begin{itemize}
    \item For CIFAR10 as the in-distribution dataset, we study SVHN, LSUN, ImageNet, and iSUN as OOD datasets.
    \item For SVHN as the in-distribution dataset, we study LSUN, ImageNet, CIFAR10 and iSUN as OOD datasets.
\end{itemize} 

We evaluate the performance on two pre-trained architectures: ResNet34~(\cite{he2016deep}) and  DenseNet~(\cite{huang2017densely}). The calibration dataset in each case is a subset of 5000 samples of the in-distribution training dataset. 

We evaluate the proposed approach, and compare it with SOTA methods based on the standard metric of probability of detection $\mathrm{P_D}$ or power (i.e., probability of correctly detecting an OOD sample) at probability of false alarm $\mathrm{P_F}$ at 0.1. Note that in some prior on OOD detection, the probability of detection is referred to as the True Negative Rate (TNR) and $1 - \mathrm{P_F}$ as the True Positive Rate (TPR), where the in-distribution samples are {considered} positives, and OOD samples are {considered} negatives.  

Recall that we focus exclusively on methods that do \textbf{not} have any outlier exposure to OOD samples, such as those in \cite{oe, odin, mahala}, and can be applied to \emph{any} pre-trained ML model. We compare our approach against baselines:  Mahalanobis~(\cite{mahala}),  Gram matrix (\cite{gram}), and Energy (\cite{energy}). For the Mahalanobis baseline, we use the scores from the penultimate layer of the network to maintain uniformity. 

To evaluate our proposed method, we systematically combine the following test statistics using our multiple testing approach as detailed in Algorithm \ref{main}: 
\begin{enumerate} 
    \item Mahalanobis distances from individual DNN layers~(\cite{mahala}): Let $g_i(X), i = 1, \ldots, L$ denote the outputs of the intermediate layers of the neural network for an input $X$. We estimate $\mu^c_i$, the class-wise mean of $g_i(.)$, as the empirical class-wise mean from the training dataset:
    \begin{equation}
        \mu^c_i = \frac{1}{n_c} \sum_{j: Y_j = c} g_i(X_j),
    \end{equation}
    where $n_c$ is the number of points with label $c$. We estimate the common covariance $\Sigma$ for all classes as 
    \begin{equation}
        \Sigma = \frac{1}{n_c} \sum_{c} \sum_{j: Y_j = c} \left(g_i(X_j) - \mu_c \right) \left(g_i(X_j) - \mu_c \right)^T.
    \end{equation}
    This is equivalent to fitting the class-
    conditional Gaussian distributions with a tied covariance. The Mahalanobis score for layer $i$ is calculated as:
    \begin{equation}
        \max_c - \left(g_i(X_j) - \mu_c \right) \Sigma^{-1} \left(g_i(X_j) - \mu_c \right)^T. 
    \end{equation}
    We calculate 5 Mahalanobis scores from the intermediate layers for the ResNet34 architecture, and 4 scores for the DenseNet architecture.
    \item Gram matrix deviations from the individual  DNN layers~(\cite{gram}): For each intermediate layer $i$, the Gram matrix of order $p$ is calculated as:
    \begin{equation}
        M^p_i(x) = \left( g^p_i {g^p_i}^T \right)^{\frac{1}{p}},
    \end{equation}
    where the power is calculated element-wise. For each flattened upper triangular Gram matrix $\overline{M^p_i}$, there are $n_i$ correlations. The class-specific minimum and maximum values for the correlation $j$ (i.e., $j$-th element of $\overline{M^p_i}$), class $c$, layer $i$ and power $p$  are estimated from the training dataset as $\text{min}[c][i][p][j]$ and  $\text{max}[c][i][p][j]$, respectively. For a new input $X$, the deviation for correlation $j$, layer $i$, power $p$ is calculated with respect to the predicted class $c_X$ as 
    \begin{equation}
        \delta_X(i,p,j) = \begin{cases}
            \frac{\overline{M^p_i(X)}[j] - \text{min}[c_X][i][p][j]}{|\text{min}[c_X][i][p][j]|} \;\;\; \text{if } \overline{M^p_i(X)}[j] >  \text{min}[c_X][i][p][j] \\ \noalign{\vskip5pt}
            \frac{\text{max}[c_X][i][p][j] - \overline{M^p_i(X)}[j] }{|\text{max}[c_X][i][p][j]|} \;\;\; \text{if } \overline{M^p_i(X)}[j] < \text{max}[c_X][i][p][j] 
            \\ \noalign{\vskip5pt}
            0 \;\;\;\;\;\;\;\;\;\;\;\;\;\;\;\;\;\;\;\;\;\;\;\;\;\;\;\;\;\;\;\;\;\;\;\text{otherwise}.
        \end{cases}
    \end{equation}
    As proposed in \cite{gram}, the Gram matrix score for layer $i$ is then calculated as the sum of $\delta_X(i,p,j)$ over values of $p$ from 1 to 10, and all values of $j$, and normalized by the empirical mean of $\delta_X(i,p,j)$. We calculate 5 Gram scores from the intermediate layers for the ResNet34 architecture, and 4 scores for the DenseNet architecture.
    \item Energy statistic~(\cite{energy}): The energy score is a temperature scaled log-sum-exponent of the softmax scores
    \begin{equation}
     -T \log \sum_{i=1}^C e^{\sigma_i(X) /T},
    \end{equation}
    where $C$ is the number of classes, $\sigma_i(.)$ are the softmax scores, and $T$ is the temperature parameter. In our experiments, we set the temperature $T$ to 100 for all in-distribution datasets, DNN architectures and OOD datasets (as stated in \cite{energy}, the energy score is not sensitive to the temperature parameter).
\end{enumerate}

We use a subset of 45000 points from the training dataset (with no overlap with the calibration dataset) to calculate the class-wise empirical means and covariance for the Mahalanobis scores, and the minimum and maximum correlations for the Gram scores.

\begin{figure}[t]
  \centering
  \subcaptionbox{{\small ResNet with CIFAR10}}[.485\linewidth][c]{%
    \includegraphics[width=1\linewidth]{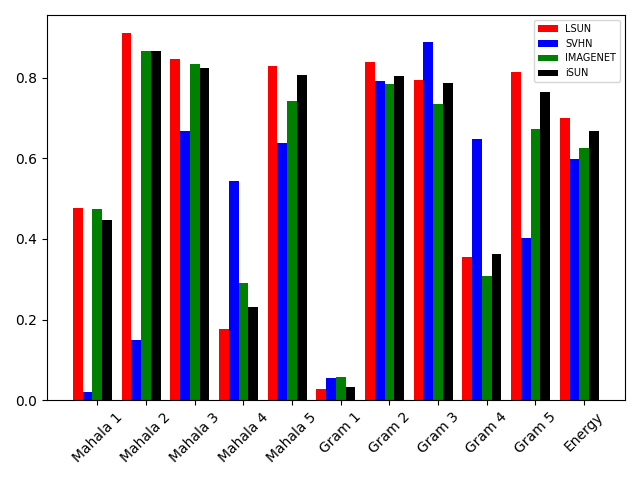}}\quad
  \subcaptionbox{{\small DenseNet with CIFAR10}}[.485\linewidth][c]{%
    \includegraphics[width=1\linewidth]{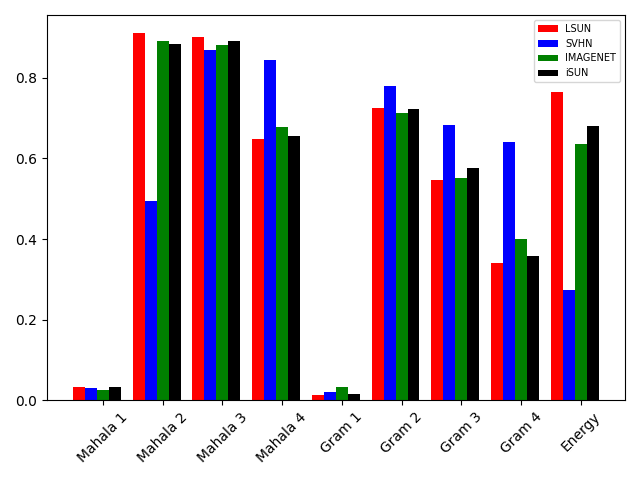}}\quad 
    \bigskip
  \subcaptionbox{{\small ResNet with SVHN}}[.485\linewidth][c]{%
    \includegraphics[width=1\linewidth]{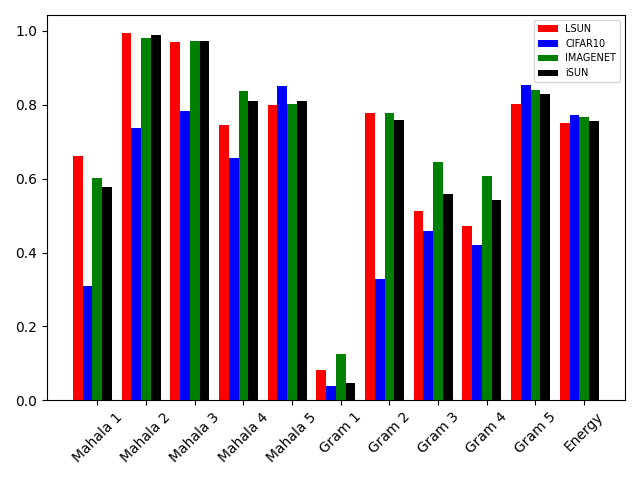}}
  \subcaptionbox{{\small DenseNet with SVHN}}[.485\linewidth][c]{%
    \includegraphics[width=1\linewidth]{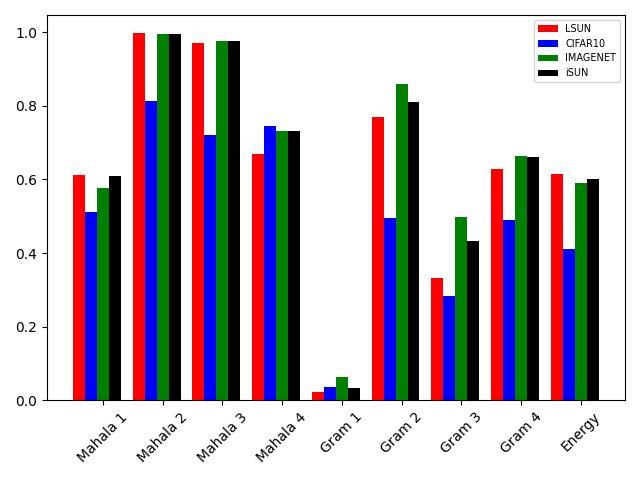}}
  \caption{Probabilities of scores rejected with CIFAR10 and SVHN as in-distribution datasets for ResNet and DenseNet}
  \label{fig:rejectde_stats}
\end{figure}

For CIFAR10 and SVHN as the in-distribution datasets, we use our proposed method in Algorithm \ref{main} to combine the Mahalanobis scores and Gram scores across layers, and the energy score, to detect OOD samples. There are 11 scores (i.e., $K = 11$) in total for the ResNet34 architecture, and 9 scores (i.e., $K = 9$) for the DenseNet architecture. Recall that Algorithm \ref{main} declares an input to be an OOD sample if any of the $K$ null hypotheses corresponding to the $K$ scores are rejected. For different OOD datasets, we empirically study the probability of each null hypotheses being rejected by Algorithm \ref{main}. In Figure \ref{fig:rejectde_stats}, we plot the empirical probability of each score $i$ being rejected, i.e., the proportion of data points in each OOD dataset for which the corresponding null hypothesis $H_{0,i}$ was rejected. The Mahalanobis score and Gram score of layer $i$ are denoted by `Mahala i' and `Gram i' respectively, and the energy score is denoted by `Energy'. We observe that while the probability of a score being rejected is high for certain OOD datsets, there exists OOD instances for which it is quite low. For example,  in the Resnet34 architecture with CIFAR10 as the in-distribution dataset, while the Mahalanobis scores of layers 2, 3 and 5, and the Gram scores of layer 5 are useful to detect OOD instances from the LSUN, Imagenet and iSUN datasets, they are not likely to be useful in detecting OOD instances from the SVHN dataset. On the other hand, the Mahalanobis and Gram scores from layer 4 of the network are more useful in detecting OOD instances from the SVHN dataset than LSUN, Imagenet and iSUN datasets. 
This study provides evidence that any single score may not be useful to detect all kinds of OOD instances that an ML model might encounter at inference time, and combining different scores systematically, as proposed in Algorithm \ref{main},  might lead to a more robust OOD detection method. We demonstrate an improvement in detection performance and the robustness of our proposed OOD detection method through extensive experiments presented further in this section.

The detection power performances for CIFAR10 and SVHN as in-distribution datasets are presented in Tables~\ref{table:cifar10_P_D} and \ref{table:svhn_P_D}, for the Mahalanobis, Gram and Energy baselines, and our proposed method of combining different statistics.  We annotate our method with the {number} of statistics used, e.g., Mahalanobis, Gram and Energy (5/4+5/4+1) uses 5,4 layers in ResNet34, DenseNet  architectures respectively, for both Mahalanobis and Gram, and the energy score. For each in-distribution dataset, we consider 8 cases, comprising of 4 OOD Datasets and 2 different DNN architectures.

\begin{table*}
\caption {{\small Comparison with  SOTA methods for CIFAR10 as in-distribution. Each entry is {\bf $\mathrm{P_D} (\%)$ at $\mathrm{P_F} = 10\%$}. We annotate our method with the no. of statistics used, e.g., Mahalanobis, Gram and Energy (5/4+5/4+1) uses 5,4 layers in ResNet34,DenseNet  architectures respectively, for both Mahalanobis and Gram, and the energy score.}}
\begin{center}
\begin{adjustbox}{width=0.6\columnwidth}
\begin{tabular}{clcccc}
\label{table:ablation-svhn}
\\ \hline \hline \\
\multirow{2}{*}{}   {\bf OOD Dataset} & {\bf Method} & {\bf ResNet34} & {\bf DenseNet}  \\
\\ \hline \hline \\
      &   Mahala (penultimate layer)                     &82.77  &{92.98}   \\
       &   Gram (sum across layers)	                                      & {96.04}  &89.97 \\
 SVHN &   Energy	                                      &{73.21}  & {42.40}   \\
       &   Ours - Mahala (5/4)                          &87.92  &{93.16}    \\
       &   Ours - Gram (5/4)                            &95.61  &89.90     \\
      &   Ours - Mahala, Energy (5/4 + 1)              &91.88  &{94.03} \\
       &   Ours - Gram, Energy (5/4 + 1)                &96.78  &90.77  \\
       &   Ours - Mahala, Gram (5/4 + 5)                &96.23  &94.21  \\
       &   Ours - Mahala, Gram and Energy (5/4+5/4+1) &{97.13}  &94.57 \\
       \hline \\
   &   Mahala (penultimate layer)                 &85.45  &82.81  \\
       &   Gram (sum across layers)	                                      &{92.34}  &{80.04}    \\
ImageNet &   Energy	                                      &{76.76}   & {94.93}    \\
       &   Ours - Mahala (5/4)                          &96.90  &95.19    \\
       &   Ours - Gram (5/4)                            &92.60  & 80.12     \\
       &   Ours - Mahala, Energy (5/4 + 1)              &97.28  &{98.09}  \\
       &   Ours - Gram, Energy (5/4 + 1)                &94.53  & 95.19 \\
       &   Ours - Mahala, Gram (5/4 + 5)                &96.38  & 92.81 \\
       &   Ours - Mahala, Gram and Energy (5/4+5/4+1) &{97.03}  & {97.20}  \\
       \hline \\
   &   Mahala  (penultimate layer)                    &90.97  & {84.11}    \\
       &   Gram (sum across layers)	                                      &{95.94}  & 81.83 \\
LSUN &   Energy	                                      &{81.16}   & {96.89}  \\
       &   Ours - Mahala (5/4)                          &{98.11}  &96.38  \\
       &   Ours - Gram (5/4)                            &96.16  &81.67  \\
       &   Ours - Mahala, Energy (5/4 + 1)              &97.87  & {98.20} \\
       &   Ours - Gram, Energy (5/4 + 1)                &96.61  &96.43   \\
       &   Ours - Mahala, Gram (5/4 + 5/4)            &98.02  &94.40  \\
       &   Ours - Mahala, Gram and Energy (5/4+5/4+1)  &{98.00} & {97.78}  \\
       \hline \\
   &   Mahala (penultimate layer)               &  {89.99}   &	 {83.19}  \\           &Gram (sum across layers)    &95.10  & 81.47    \\
iSUN &   Energy	                                      &{80.11}  & {95.10}  \\
       &   Ours - Mahala (5/4)                          &97.24  &95.26   \\
       &   Ours - Gram (5/4)                            &95.11  &81.09    \\
       &   Ours - Mahala, Energy (5/4 + 1)              &{97.17}  & {97.12}     \\
       &   Ours - Gram, Energy (5/4 + 1)                &96.19  & 94.73   \\
       &   Ours - Mahala, Gram (5/4 + 5/4)            &97.36 & 92.93  \\
      &   Ours - Mahala, Gram and Energy (5/4+5/4+1) &{97.67}  &96.34 \\
       \hline
       \hline 
\end{tabular}
\end{adjustbox}
\end{center}
\label{table:cifar10_P_D}
\end{table*}

\begin{table*}
\caption {{\small Comparison with  SOTA methods for SVHN as in-distribution. Each entry is {\bf $\mathrm{P_D} (\%)$ at $\mathrm{P_F} = 10\%$}. We annotate our method with the no. of statistics used, e.g., Mahalanobis, Gram and Energy (5/4+5/4+1) uses 5,4 layers in ResNet34,DenseNet  architectures respectively, for both Mahalanobis and Gram, and the energy score.}}
\begin{center}
\begin{adjustbox}{width=0.6\columnwidth}
\begin{tabular}{clcccc}
\\ \hline \hline \\
\multirow{2}{*}{}   {\bf OOD Dataset} & {\bf Method} & {\bf ResNet34} & {\bf DenseNet}  \\
\\ \hline \hline \\
    &   Mahala (penultimate layer)                 &96.12  &96.34  \\
      &   Gram (sum across layers)	                &97.52  &93.57   \\
ImageNet  &   Energy	                                      &85.14 &70.53   \\
       &   Ours - Mahala (5/4)                          &99.91 &99.95   \\
       &   Ours - Gram (5/4)                            &97.68 &94.38     \\
       &   Ours - Mahala, Energy (5/4 + 1)              &99.89  &99.93   \\
       &   Ours - Gram, Energy (5/4 + 1)                &97.85  &95.01  \\
       &   Ours - Mahala, Gram (5/4 + 5/4)                &99.83   &99.91 \\
      &   Ours - Mahala, Gram and Energy (5/4+5/4+1) 
       &99.84   &99.89  \\
       \hline \\
 &   Mahala  (penultimate layer)                    &93.74   & 94.17  \\
 &   Gram (sum across layers)	            
      & 96.20   &88.25 \\
     LSUN &   Energy	                                   &81.30   &71.36  \\
       &   Ours - Mahala (5/4)                          &99.98   &100.0   \\
       &   Ours - Gram (5/4)                            &96.54   &89.02   \\
        &   Ours - Mahala, Energy (5/4 + 1)              &99.96   &99.99 \\
       &   Ours - Gram, Energy (5/4 + 1)                &96.82   &90.56  \\
       &   Ours - Mahala, Gram (5/4 + 5/4)            &99.96   &99.98 \\
       &   Ours - Mahala, Gram and Energy (5/4+5/4+1)  &99.95   &100.0  \\
       \hline \\
  &   Mahala (penultimate layer)                &95.23   &96.01 \\
     &Gram (sum across layers)    
    &96.50   &91.46   \\
 iSUN  &   Energy	                                &82.79   &71.20    \\
     &   Ours - Mahala (5/4)                         &99.98   &100.0  \\
      &   Ours - Gram (5/4)                            &96.80   &91.89   \\
       &   Ours - Mahala, Energy (5/4 + 1)              &99.93   &100.0  \\
       &   Ours - Gram, Energy (5/4 + 1)                &97.21   &92.69  \\
       &   Ours - Mahala, Gram (5/4 + 5/4)            &99.88   &99.98 \\
       &   Ours - Mahala, Gram and Energy (5/4+5/4+1) 
      & 99.88   &99.98 \\
       \hline \\
 &   Mahala (penultimate layer)                     &96.09  & 94.25   \\
       &   Gram (sum across layers)	            
       &91.58   &69.77    \\
  CIFAR10 &   Energy	                              &83.31  &54.07   \\
      &   Ours - Mahala (5/4)                
       &98.31   &97.64  \\
      &   Ours - Gram (5/4)                  
       &92.39   &72.84     \\
        &   Ours - Mahala, Energy (5/4 + 1)              &98.13  &97.16 \\
       &   Ours - Gram, Energy (5/4 + 1)                &92.91  &78.03 \\
       &   Ours - Mahala, Gram (5/4 + 5)                &97.15  &94.83  \\
      &   Ours - Mahala, Gram and Energy (5/4+5/4+1) &97.35  &95.23   \\
       \hline \hline \\
\end{tabular}
\end{adjustbox}
\end{center}
\label{table:svhn_P_D}
\end{table*}

\begin{enumerate}
    \item \textbf{Improvement in probability of detection across OOD datasets and DNN architectures:} The best probability of detection in all 8 cases with CIFAR10 as in-distribution correspond to our method of combining statistics. 
    Similarly, with SVHN as in-distribution, our method of combining statistics gives the best probability of detection in all 8 cases.
    
    Thus, {our approach leads to an improvement} across OOD datasets and DNN architectures.

    \item \textbf{Lower variation in detection probability across OOD datasets and DNN architectures:} {Detection probabilities of baselines Mahalanobis, Gram and Energy exhibit a much higher variation across different kinds of OOD samples} as compared to the combination of all statistics.
    
    With CIFAR10 as the in-distribution dataset, for the ResNet34 architecture: the variation in $\mathrm{P_D}$ is $82.77-90.97$ for the Mahalanobis baseline, $92.34-96.04$ for the Gram baseline, and $73.21-81.16$ for the energy baseline. In contrast, our method of combining all statistics has a variation of $97.03-98.00$. For DenseNet, the variation in $\mathrm{P_D}$ is $82.81-92.98$ for the Mahalanobis baseline, $80.04-89.97$ for the  Gram baseline, and $42.40-96.89$ for the energy baseline. Our method of combining all statistics has a variation of $94.57-97.78$. 
    
    A similar trend is seen with SVHN as the in-distribution dataset. Our method reduces the variation across different kinds of OOD samples by almost 5X. This is a key improvement, as the kind of OOD samples encountered at inference time is unknown, and our proposed method shows very little variation across different OOD datasets.

    \item \textbf{Impact of combining all the scores:} For CIFAR10 as the in-distribution dataset, in all 8 cases, combining all the scores - Mahalanobis and gram from individual layers, and the energy score, is either the best method, or within $1\%$ of the best performance. 
    
    Similarly, with SVHN as the in-distribution dataset, in 7 out of 8 cases, combining all the scores is either the best method, or within $1\%$ of the best performance (the gap is $2.41\%$ in the remaining case). 
    
    Thus, in contrast to existing methods,  {combining all the statistics using Algorithm \ref{main} is robust to different kinds of OOD samples across DNN architectures.}
    
    
\end{enumerate}

\begin{table*}[ht]
\caption {{\small Comparison with naive averaging rule for CIFAR10 and SVHN as in-distribution. Each entry is {\bf $\mathrm{P_D} (\%)$ at $\mathrm{P_F} = 10\%$}.}}
\begin{center}
\begin{adjustbox}{width=0.6\columnwidth}
\begin{tabular}{clccccc}
\\ \hline \hline \\
\multirow{2}{*}{}  {\bf In-distribution Dataset} & {\bf OOD Dataset} & {\bf Method} & {\bf ResNet34} & {\bf DenseNet}  \\
\\ \hline \hline \\
  
CIFAR10      &SVHN &   Naive                             &{81.13}  & 83.28   \\
      & &   Ours   &97.13 &{94.57}    \\
       \hline \\

CIFAR10      &ImageNet &   Naive                           &{86.45}   & {80.96}    \\
      & &   Ours           &97.03  &97.20    \\
       \hline \\
       
CIFAR10      &LSUN &   Naive                                 &{91.31}   & {83.79}  \\
      & &  Ours            &{98.00}  &97.78  \\
       \hline \\

CIFAR10 &iSUN &   Naive                               &{89.22}  & {81.70}  \\
      & &   Ours   &97.67  &96.34   \\
       \hline
       \hline \\
 
SVHN     &ImageNet  &   Naive                                &97.08 &95.67   \\
      & &   Ours       &99.84 &99.89   \\
       \hline \\

     SVHN  &LSUN &   Naive                              &95.00   &92.81  \\
     &  &   Ours    &99.95   &100.0   \\
        \hline \\

SVHN    & iSUN  &   Naive                         &96.00   &94.53    \\
     &   &   Ours         &99.88   &99.98  \\
        \hline \\

    SVHN  & CIFAR10 &   Naive  	                              &86.10  &77.22   \\
    &  &   Ours                   
       &97.35   &95.23  \\
       \hline \hline \\
\end{tabular}
\end{adjustbox}
\end{center}
\label{table:svhn_and_cifar10_naive}
\end{table*}

\begin{table*}[h]
\caption {{\small Comparison with  SOTA OOD detection techniques for CIFAR10 as in-distribution dataset. Each entry is {\bf AUROC} for the corresponding detection method, OOD dataset and DNN architecture.}}
\begin{center}
\begin{adjustbox}{width=0.6\columnwidth}
\begin{tabular}{clcccc}
\\ \hline \\
\multirow{2}{*}{}  {\bf OOD Dataset} & {\bf Method} & {\bf ResNet34} & {\bf DenseNet}  \\
\\ \hline \\
SVHN   &   Mahala (penultimate layer)                               &93.86  & 96.72  \\
      &   Gram (sum across layers)	                                      &97.28  &94.31    \\
      &   Energy	                                      &90.24  &77.92    \\
      &   Ours - Mahala (5/4)                          &95.34  & 96.70   \\
      &   Ours - Gram (5/4)                            &97.47  & 94.28    \\
      &   Ours - Mahala, Energy (5/4 + 1)              &95.84  &96.99 \\
      &   Ours - Gram, Energy (5/4 + 1)                &97.90  & 96.20   \\
      &   Ours - Mahala, Gram (5/4 + 5/4)         &97.56 & 96.98  \\
      &   Ours - Mahala, Gram and Energy (5/4+5/4+1) &97.72  & 97.24  \\
      \hline \\
ImageNet   &   Mahala (penultimate layer)                 &94.84  &93.12  \\
      &   Gram 	(sum across layers)                                      &95.90  &89.83   \\
      &   Energy	                                      &91.40  &96.03    \\
      &   Ours - Mahala (5/4)                          &97.89  &97.32  \\
      &   Ours - Gram (5/4)                            &96.09  &89.75    \\
      &   Ours - Mahala, Energy (5/4 + 1)              &97.97  &98.13   \\
      &   Ours - Gram, Energy (5/4 + 1)                &97.07  & 96.79  \\
      &   Ours - Mahala, Gram (5/4 + 5/4)                &97.55  & 96.67 \\
      &   Ours - Mahala, Gram and Energy (5/4+5/4+1) &97.63  & 97.70  \\
      \hline \\
LSUN   &   Mahala  (penultimate layer)                    &96.28  &90.00   \\
      &   Gram 	  (sum across layers)                    &97.31  & 87.97  \\
      &   Energy	                                      &92.35  &96.83    \\
      &   Ours - Mahala (5/4)                          &98.20  &97.54  \\
      &   Ours - Gram (5/4)                            &97.46  & 87.76   \\
      &   Ours - Mahala, Energy (5/4 + 1)              &98.07  & 98.16 \\
      &   Ours - Gram, Energy (5/4 + 1)                &97.76  & 97.14  \\
      &   Ours - Mahala, Gram (5/4 + 5/4)            &97.99  & 96.82   \\
      &   Ours - Mahala, Gram and Energy (5/4+5/4+1)  &97.96 & 97.74  \\
      \hline \\
iSUN   &   Mahala  (penultimate layer)                    &96.07  &93.71    \\
      &   Gram 	  (sum across layers)                   &97.01  & 90.48    \\
      &   Energy	                                      &92.05  & 96.25    \\
      &   Ours - Mahala (5/4)                          &97.95  & 97.39  \\
      &   Ours - Gram (5/4)                            &97.15  & 90.36   \\
      &   Ours - Mahala, Energy (5/4 + 1)              &97.93  & 97.89 \\
      &   Ours - Gram, Energy (5/4 + 1)                &97.66  & 96.71  \\
      &   Ours - Mahala, Gram  (5/4 + 5/4)            &97.79  & 96.76 \\
      &   Ours - Mahala, Gram and Energy (5/4+5/4+1)  &97.83 &  97.47  \\
      \hline \\
\end{tabular}
\end{adjustbox}
\end{center}
\label{table:auroc_cifar10}
\end{table*}

\begin{table*}[h]
\caption {{\small  Comparison with  SOTA OOD detection techniques for SVHN as in-distribution dataset. Each entry is {AUROC} for the corresponding detection method, OOD dataset and DNN architecture.}}
\begin{center}
\begin{adjustbox}{width=0.6\columnwidth}
\begin{tabular}{clcccc}
\\ \hline \\
\multirow{2}{*}{}  {\bf OOD Dataset} & {\bf Method} & {\bf ResNet34} & {\bf DenseNet}  \\
\\ \hline \\
LSUN   &   Mahala (penultimate layer)               &96.06 &96.22  \\
      &   Gram 	                                      &97.23 &94.17   \\
      &   Energy	                               &87.58 &86.01   \\
      &   Ours - Mahala (5/4)                          &99.00 &98.92   \\
      &   Ours - Gram (5/4)                            &97.19 &94.11    \\
      &   Ours - Mahala, Energy (5/4 + 1)              &98.76 &98.94 \\
      &   Ours - Gram, Energy (5/4 + 1)                &97.47 & 95.69  \\
      &   Ours - Mahala, Gram  (5/4 + 5/4)         &98.82 &99.08 \\
      &   Ours - Mahala, Gram and Energy (5/4+5/4+1) &98.21 &99.06  \\
      \hline \\
ImageNet   &   Mahala (penultimate layer)                 &96.81 &97.01   \\
      &   Gram 	                                      &97.75 &96.34  \\
      &   Energy	                                      &90.33 &85.76    \\
      &   Ours - Mahala (5/4)                          &98.99 &98.89    \\
      &   Ours - Gram (5/4)                            &97.73 &96.32   \\
      &   Ours - Mahala, Energy (5/4 + 1)              &98.79 & 98.91 \\
      &   Ours - Gram, Energy (5/4 + 1)                &98.01 &97.11  \\
      &   Ours - Mahala, Gram  (5/4 + 5/4)                &98.87  &99.04  \\
      &   Ours - Mahala, Gram and Energy (5/4+5/4+1) & 98.25 & 99.02 \\
      \hline \\
iSUN   &   Mahala  (penultimate layer)                    &96.49  & 96.85    \\
      &   Gram 	                                      &97.40  & 95.47    \\
      &   Energy	                                      &88.75  &  85.69  \\
      &   Ours - Mahala (5/4)                          &98.99  & 98.91  \\
      &   Ours - Gram (5/4)                            &97.37  & 95.42  \\
      &   Ours - Mahala, Energy (5/4 + 1)              &98.76 & 98.94 \\
      &   Ours - Gram, Energy (5/4 + 1)                &97.63 & 96.40  \\
      &   Ours - Mahala, Gram (5/4 + 5/4)            &98.82 & 99.07  \\
      &   Ours - Mahala, Gram and Energy (5/4+5/4+1)  &98.21 & 99.05 \\
      \hline \\
CIFAR10   &   Mahala  (penultimate layer)                    &96.90 & 96.59   \\
      &   Gram 	                                      &95.35 & 87.06   \\
      &   Energy	                                      &89.09 &  77.72  \\
      &   Ours - Mahala (5/4)                          &97.63 & 97.50 \\
      &   Ours - Gram (5/4)                            &95.35 & 87.21  \\
      &   Ours - Mahala, Energy (5/4 + 1)              &97.68 & 97.43 \\
      &   Ours - Gram, Energy (5/4 + 1)                &95.94 & 91.15  \\
      &   Ours - Mahala, Gram (5/4 + 5/4)            &97.32 & 96.91 \\
      &   Ours - Mahala, Gram and Energy (5/4+5/4+1)  &97.10 & 96.98  \\
      \hline \\
\end{tabular}
\end{adjustbox}
\end{center}
\label{table:auroc_svhn}
\end{table*}

We further compare our proposed method of combining multiple scores in Algorithm \ref{main} with a baseline that combines scores naively through an averaging rule. 
This naive OOD detection test maintains thresholds $\tau_1, \ldots, \tau_K$ for the $K$ scores. Let $\gamma_i$ be the weight for the $i$-th score where
\begin{equation}
    \gamma_i = \identityf{T_i \geq \tau_i},
\end{equation}
and let $\gamma$ be defined as 
\begin{equation}
    \gamma = \frac{1}{K} \sum_{i=1}^K \gamma_i.
\end{equation}
The naive averaging OOD detection rule declares an input to be an OOD sample if $\gamma \geq \frac{1}{2}$. The thresholds $\tau_1, \ldots, \tau_K$ are set to ensure a false alarm probability of $0.1$. In Table \ref{table:svhn_and_cifar10_naive}, we present a comparison of the detection powers of the naive averaging rule and our proposed method of combining scores, where both methods use the Mahalanobis and Gram scores from all the layers, and the energy score. We observe that the naive averaging method does not perform as well as our proposed method of combining statistics, and indeed has a high variability across different OOD datasets. Thus, we see that while it is imperative to combine multiple scores for effective and robust OOD detection, combining them in an adhoc manner such as uniform averaging does not yield good results. 

In some of the prior work on OOD detection, the Area Under the Receiver Operating Characteristic  (AUROC) metric has been used to compare different tests (\cite{odin, energy, gram, mahala}). However, it is not clear that this measure is useful in such a comparison, especially when the ROC is being estimated through simulations. It is possible for a test (say, Test 1) to have a larger AUROC than another test (say, Test 2), with Test 2 having a larger detection power than Test 1 for all values of false alarm less than some threshold (equivalently, all values of TPR greater than some threshold). Nevertheless, we provide the AUROC numbers for our experimental setups below for completeness.

Table \ref{table:auroc_cifar10} contains the AUROC numbers for CIFAR10 as the in-distribution dataset, and Table \ref{table:auroc_svhn} contains the AUROC numbers for SVHN as the in-distribution dataset. We observe similar patterns in the AUROC numbers as the above observations on the detection power at a fixed false alarm probability. The Mahalanobis, Gram and Energy baselines have a high variability across different kinds of OOD samples and DNN architectures, whereas our proposed method of combining all statistics has a low variability. For instance, the Mahalanobis, Gram and Energy baselines for DenseNet architecture with CIFAR10 as the in-distribution dataset have an average variability of 10.5$\%$ in their AUROC, whereas our proposed method of combining all statistics has a variability of $0.5\%$ in the AUROC performance. Our proposed method of combining all statistics either has the best AUROC performance or within $1\%$ of the best performance in all 8 cases for CIFAR10 and SVHN as the in-distribution datasets.


\section{Conclusion}
While empirical methods for OOD detection have been studied extensively in recent literature, a formal
characterization of OOD is lacking. We proposed a characterization for the notion of OOD that includes both the input distribution and the ML model. This provided insights for the construction of effective OOD detection tests.
Our approach, inspired by \emph{multiple hypothesis testing}, 
allows us to systematically combine any number of different statistics derived from the ML model with an arbitrary dependence structure. 

Furthermore, our analysis allows us to set the test thresholds to meet given constraints on the probability of incorrectly classifying an in-distribution sample as OOD (false alarm probability). We provide strong theoretical guarantees on the probability of false alarm in OOD detection, conditioned on the dataset used for computing the conformal p-values.

In our experiments, we observe that no single score is useful for detecting different kinds of OOD instances. We demonstrated that our proposed method outperforms threshold-based tests for OOD detection proposed in prior work. Across different kinds of OOD examples, we observed that the state-of-the-art methods from prior work exhibit high variability across OOD instances and neural network architectures in their probability of detection of OOD samples. In contrast, our proposed method is robust and provides uniformly good performance (with respect to both detection power and AUROC) across different kinds of OOD samples and neural network architectures. This robustness is important, since a useful OOD detection algorithm should perform well regardless of the type of OOD instance encountered at inference time.


\vfill

\bibliographystyle{IEEEtran}
\bibliography{ref.bib}

\appendix



\subsection{Preliminaries on Multiple Testing}

Multiple hypothesis testing (a.k.a. multiple testing) refers to the inference problem testing between multiple binary hypotheses, e.g., $\mathrm{H}_{0,i}$ versus $\mathrm{H}_{1,i}$, $i=1,2, \ldots,K$. 
For {a given} multiple testing procedure, let $R$ be the number of null hypotheses rejected (i.e., number of tests declared as the alternative $H_{1,i}$), out of which $V$ is the number of true null hypotheses. Some measures of performance for multiple testing procedures are as follows: 

\begin{enumerate}
    \item Family Wise Error Rate (FWER): The probability of rejecting at least one null hypothesis when all of them are true. 
    \item False Discovery Rate (FDR): Expected ratio of number of true null hypotheses rejected ($V$) and the total number of hypotheses rejected ($R$), i.e.,
    \begin{equation}
        \mathrm{FDR} = \mathrm{E}\left[\frac{V}{R} \identityf{R > 0}\right].
    \end{equation}
    where the expectation is taken over the joint distribution of the statistics involved in the multiple testing problem.
\end{enumerate}
When all the null hypotheses are true, $V=R$ with probability 1, and:
$$
\mathrm{FDR} =  \mathrm{E}\left[\identityf{R > 0}\right] = \mathrm{P} \{R > 0\} = \text{FWER}.
$$

Various multiple testing procedures have been proposed in literature depending on the quantity of interest to be controlled. 
Widely used multiple testing procedures involve calculating the p-values for each test $i$ as ${Q}^i$, and combining these p-values to give decisions for each hypothesis. Let $\alpha > 0$. One of the earliest tests proposed to control the FWER is the Bonferroni test. In this test, each ${Q}^i$ is computed, and for each $i = 1, \ldots, K$, the corresponding hypothesis $\mathrm{H}_{0,i}$ is rejected if 
\begin{equation}
    {Q}^i \leq \frac{\alpha}{K}.
\end{equation}
This test controls the FWER at $\alpha$ for any joint distribution of the test statistics of the $K$ hypotheses. However, the power of this test has been observed to be low, and hence the test is considered to be conservative. 
The FDR measure was proposed by \cite{bh1995}, who also proposed a procedure to control the FDR. Let the p-values for each test $i$ be ${Q}^i$, and let the ordered p-values be denoted by ${Q}^{(1)}, {Q}^{(2)}, \ldots, {Q}^{(K)}$. Let 
\begin{equation}
    m = \max \left\{i : {Q}^{(i)} \leq \frac{\alpha i}{K} \right\}.
\end{equation}
The Benjamini-Hochberg (BH) procedure rejects hypotheses $\mathrm{H}_{0,1}, \ldots, \mathrm{H}_{0,m}$, and controls the FDR at level $\alpha$ when the test statistics are independent. 
\cite{benjamini2001} showed that the constants in the BH procedure can be modified to $\frac{\alpha i}{K\sum_{j=1}^K \frac{1}{j}}$ instead of $\frac{\alpha i}{K}$ to control the FDR at level $\alpha$ for arbitrarily dependent test statistics. 
Note that the Bonferroni procedure and the BH procedure can be used to test against the global null $\mathrm{H_0}$ (all $\mathrm{H_{0,i}}$ are true), where the probability of false alarm $\mathrm{P_{H_0}}( \text{reject } \mathrm{H_0})$ is equal to the FWER and FDR. Our proposed algorithm for the OOD detection problem builds on the BH procedure with the modified constants, and conformal p-values calculated using a calibration dataset $\mathcal{T}_\mathrm{cal}$, where we aim to control the conditional probability of false alarm $\mathrm{P_F}(\mathcal{T}_\mathrm{cal}) = \mathrm{P_{H_0}}( \text{reject } \mathrm{H_0} | \mathcal{T}_\mathrm{cal})$ with high probability.

\subsection{Proposed OOD Modelling}

In Section 2, we conclude that functions of the input from the learning algorithm apart from the final output are required for the OOD formulation presented above. Note that this does not violate the data-processing inequality, as the out-distribution $\mathsf{P}_{X,\hat{Y}}$ characterizes the input and the model, and these functions of the input give us additional information regarding the learning algorithm. In addition, these functions give us information to differentiate between the null and the alternate hypothesis.

\subsection{Proof of Theorem 1}

For $\ell = 1, 2, \ldots, K$, let 
\begin{equation}
    \alpha_\ell = \frac{\alpha \ell}{C(K)K},
\end{equation}
where 
\begin{equation}
    C(K) = (1 + \epsilon)\sum_{j=1}^K \frac{1}{j}
\end{equation}
and let $\alpha_0 = 0$.
As in \eqref{eq:define_P_F}, the probability of false alarm conditioned on the calibration set $\mathcal{T}_\mathrm{cal}$ is given by 
\begin{equation}
    \mathrm{P}_\mathrm{F} (\mathcal{T}_\mathrm{cal}) = \mathrm{P_{H_0}}(\text{reject } \mathrm{H}_0 | \mathcal{T}_\mathrm{cal}) = \mathrm{P_{H_0}}(m \geq 1 | \mathcal{T}_\mathrm{cal}), 
\end{equation}
where $m$ is as defined in Algorithm 1. Here $\mathrm{H}_0$ denotes the global null hypothesis, which corresponds to all the $\mathrm{H}_{0,i}$ being true. Note that $m \geq 1$ signifies that $\mathrm{H}_{0,(1)}, \ldots, \mathrm{H}_{0,(m)}$ are being rejected. Let
\[
A_\ell = \{\text{exactly $\ell$ of the $\mathrm{H}_{0,i}$'s are rejected}\}.
\]
Then, 
\begin{equation}\label{eq:P_F_expansion}
    \mathrm{P}_\mathrm{F} (\mathcal{T}_\mathrm{cal}) = \sum_{\ell=1}^K \mathrm{P_{H_0}} (A_\ell | \mathcal{T}_\mathrm{cal}).
\end{equation}
The following lemma is useful in deriving an upper bound for $\mathrm{P}_\mathrm{F}$. 
\begin{lemma} \label{lem:lem3}
For $\ell = 1, \ldots, K$, 
\begin{equation}
    \mathrm{P_{H_0}}(A_\ell) = \frac{1}{\ell} \sum_{i=1}^{K} \mathrm{P_{H_0}}(\{\hat{Q}^i \leq \alpha_\ell\} \cap A_\ell), 
\end{equation}
where $\hat{Q}^i$ is as defined in Section 3. 
\end{lemma}
\begin{proof}

Let
\[
\mathcal{W}_\ell = \{\text{all subsets of $\{1,2,\ldots,K\}$ with $\ell$ elements}\}.
\]
Let $A_\ell^\mathcal{V}$ be the subset of $A_\ell$ where the $\ell$ null hypotheses rejected  correspond to the indices in $\mathcal{V} \in \mathcal{W}_\ell$. Then
\begin{equation}
    A_\ell = \bigcup\limits_{\mathcal{V} \in \mathcal{W}_\ell} A_\ell^{\mathcal{V}}.
\end{equation}
Note that if $\ell$ null hypotheses corresponding to the indices in $\mathcal{V} \in \mathcal{W}_\ell$ are rejected, then the conformal p-values corresponding to these $\ell$ tests are less than or equal to $\alpha_\ell$ (since the maximum among them is less than or equal to $\alpha_\ell$), and the conformal p-values corresponding to the remaining $K - \ell$ tests are greater than $\alpha_\ell$, i.e., 
\begin{equation}
    \hat{Q}^i \leq \max\limits_{j \in \mathcal{V}} \hat{Q}^j \leq  \alpha_\ell \;\;\; \text{ for } i \in \mathcal{V},
\end{equation}
and 
\begin{equation}
    \hat{Q}^i > \alpha_\ell \;\;\;\;\;\;\;\;\;\;\;\;\;\;\;\;\;\;\;\; \text{ for } i \not\in \mathcal{V}.
\end{equation}
Thus,
\begin{equation}\label{eq:prob_intersection_A_l_Q_i}
    \mathrm{P_{H_0}}(\{\hat{Q}^{i} \leq \alpha_\ell\} \cap A_\ell^\mathcal{V} ) = \begin{cases}
            \mathrm{P_{H_0}}(A_\ell^\mathcal{V}) \;\;\;\;\; \text{if } i \in \mathcal{V} \\
            0 \;\;\;\;\;\;\;\;\;\;\;\;\;\;\;\;\;\text{else}.
        \end{cases}
\end{equation}
Then, 
\begin{align}
    \sum_{i=1}^{K} \mathrm{P_{H_0}}(\{\hat{Q}^{i} \leq \alpha_\ell\} \cap A_\ell) 
    &=\sum_{i=1}^K \sum_{\mathcal{V} \in \mathcal{W}_\ell}  \mathrm{P_{H_0}}(\{\hat{Q}^{i} \leq \alpha_\ell \} \cap A_\ell^\mathcal{V}) \\
    &= \sum_{\mathcal{V} \in \mathcal{W}_\ell} \sum_{i=1}^K \mathrm{P_{H_0}}(\{\hat{Q}^i \leq \alpha_v \} \cap A_\ell^\mathcal{V}) \\
    &= \sum_{\mathcal{V} \in \mathcal{W}_\ell}\sum_{i=1}^K \identityf{i \in \mathcal{V}} \mathrm{P_{H_0}}(A_\ell^\mathcal{V}) \\
    &= \sum_{\mathcal{V} \in \mathcal{W}_\ell} \mathrm{P_{H_0}} (A_\ell^\mathcal{V}) \sum_{i=1}^K \identityf{i \in \mathcal{V}} \\
    &= \sum_{\mathcal{V} \in \mathcal{W}_\ell} \mathrm{P_{H_0}}(A_\ell^\mathcal{V})\, \ell \\
    &= \ell \, \mathrm{P_{H_0}}(A_\ell),
\end{align}
where the first equality arises from the fact that $A_\ell$ is the union of disjoint sets $ A_\ell^\mathcal{V} $ for $\mathcal{V} \in \mathcal{W}_\ell$, and the third equality follows from \eqref{eq:prob_intersection_A_l_Q_i}.
\end{proof}

Using the result from Lemma~\ref{lem:lem3} in the expression for $\mathrm{P}_\mathrm{F}(\mathcal{T}_\mathrm{cal})$ in \eqref{eq:P_F_expansion}, we obtain that
\begin{equation}\label{eq:P_F_A_l}
    \mathrm{P}_\mathrm{F}(\mathcal{T}_\mathrm{cal}) = \sum_{\ell=1}^K \sum_{i=1}^K \frac{1}{\ell} \mathrm{P_{H_0}}(\{\hat{Q}^i \leq \alpha_\ell\} \cap A_\ell | \mathcal{T}_\mathrm{cal}). 
\end{equation}
Note that by definition, $\alpha_0 < \alpha_1 < \ldots, \alpha_K$. Thus, 
\begin{equation}
    \{\hat{Q}^i \leq \alpha_\ell\} \cap A_{\ell} = \cup_{j=1}^\ell \{\hat{Q}^i \in (\alpha_{j-1}, \alpha_j]\} \cap A_{\ell},
\end{equation}
and 
\begin{align}
    \mathrm{P}_\mathrm{F}(\mathcal{T}_\mathrm{cal}) &=  \sum_{i=1}^K  \sum_{\ell=1}^K \frac{1}{\ell} \sum_{j=1}^\ell \mathrm{P_{H_0}}(\{\hat{Q}^i \in (\alpha_{j-1},\alpha_j]\} \cap A_{\ell} | \mathcal{T}_\mathrm{cal}) \\
     &=  \sum_{i=1}^K  \sum_{j=1}^K \sum_{\ell=j}^K \frac{1}{\ell}\mathrm{P_{H_0}}(\{\hat{Q}^i \in (\alpha_{j-1},\alpha_j]\} \cap A_{\ell} | \mathcal{T}_\mathrm{cal}) \\
    &\leq \sum_{i=1}^K  \sum_{j=1}^K \sum_{\ell=j}^K \frac{1}{j} \mathrm{P_{H_0}}(\{\hat{Q}^i \in (\alpha_{j-1},\alpha_j]\} \cap A_{\ell} | \mathcal{T}_\mathrm{cal}) \\
    &\leq \sum_{i=1}^K  \sum_{j=1}^K \frac{1}{j} \sum_{\ell=1}^K  \mathrm{P_{H_0}}(\{\hat{Q}^i \in (\alpha_{j-1},\alpha_j]\} \cap A_{\ell} | \mathcal{T}_\mathrm{cal}).\label{eq:P_F_final}
\end{align}
 Note that the events $A_{\ell}$ are disjoint for $\ell = 1, \ldots, K$. 
Thus, 
\begin{align}
    \sum_{\ell=1}^K \mathrm{P_{H_0}}(\{\hat{Q}^i \in (\alpha_{j-1}, \alpha_j]\} \cap A_{\ell} | \mathcal{T}_\mathrm{cal}) &= 
    \mathrm{P_{H_0}}(\{\hat{Q}^i \in (\alpha_{j-1}, \alpha_j]\} \cap (\cup_{\ell=1}^K A_{\ell}) | \mathcal{T}_\mathrm{cal}) \\
    &\leq \mathrm{P_{H_0}}(\{\hat{Q}^i \in (\alpha_{j-1}, \alpha_j]\} | \mathcal{T}_\mathrm{cal}) \\
    &= \mathrm{P_{H_0}}(\{\hat{Q}^i \leq \alpha_j\} | \mathcal{T}_\mathrm{cal}) - \mathrm{P}(\{\hat{Q}^i \leq \alpha_{j-1}\} | \mathcal{T}_\mathrm{cal})\label{eq:split_prob}.
\end{align} 
Using this in \eqref{eq:P_F_final},we get that
\begin{align}
    \mathrm{P}_\mathrm{F}(\mathcal{T}_\mathrm{cal}) 
    &\leq \sum_{i=1}^K \sum_{j=1}^K \frac{1}{j} (\mathrm{P_{H_0}}(\{\hat{Q}^i \leq \alpha_j\} | \mathcal{T}_\mathrm{cal}) - \mathrm{P_{H_0}}(\{\hat{Q}^i \leq \alpha_{j-1}\} | \mathcal{T}_\mathrm{cal})).
\end{align}

Let $r_j^i = \mathrm{P_{H_0}}(\{\hat{Q}^i \leq \alpha_j\} | \mathcal{T}_\mathrm{cal})$. Then, rearranging the terms from above, we get
\begin{equation}
    \mathrm{P}_\mathrm{F}(\mathcal{T}_\mathrm{cal})  = \sum_{i=1}^K \left[ \sum_{j=1}^{K-1} \frac{r_j^i}{j(j+1)} + \frac{r_K^i}{K} \right].
\end{equation}
Note that $\mathrm{P}_\mathrm{F}$ is a function of $\mathcal{T}_\mathrm{cal}$  only through random variables $r_j^i$. We have from \cite{conformal, bates} that~$r_j^i$ follows a Beta distribution, i.e.,~$r_j^i \sim \mathrm{Beta} (a_j, b_j)$, where 
\begin{align}
    a_j &= \left\lfloor (n_\mathrm{cal} + 1)\frac{\alpha j}{C(K) K} \right\rfloor \\
    b_j &= (n_\mathrm{cal} + 1) - a_j.
\end{align}
The mean of this distribution is $\mu_j = \frac{a_j}{a_j + b_j}$. 
Let $E$ denote the event 
\begin{equation}
\bigcap_{i=1}^K \bigcap_{j=1}^K \left\{r_j^i \leq (1 + \epsilon) \frac{\alpha j}{C(K) K} \right\}.
\end{equation}
When the condition on $n_\mathrm{cal}$ in Lemma \ref{cal_assumption} is satisfied, we have that 
\begin{equation}
    \mathrm{P}(E) \geq 1 - \delta.
\end{equation}
Under the event $E$, we have that 
\begin{align}
    \mathrm{P}_\mathrm{F}(\mathcal{T}_\mathrm{cal}) &= \sum_{i=1}^K \left[ \sum_{j=1}^{K-1} \frac{r_j^i}{j(j+1)} + \frac{r_K^i}{K} \right] \\
    &\leq \sum_{i=1}^K \left[\sum_{j=1}^{K-1} \frac{(1+\epsilon)\alpha}{(j+1)C(K)K} + \frac{(1+\epsilon)\alpha}{C(K)K} \right] \\
    &= \frac{(1+\epsilon)\alpha}{C(K)} \left[ \sum_{j=1}^{K-1} \frac{1}{j+1} + 1 \right] \\
    &= \frac{(1+\epsilon)\alpha}{C(K)} \left( \sum_{j=1}^{K} \frac{1}{j}\right) = \alpha.
\end{align}

Thus, with probability greater than $1 - \delta$, we have that 
\begin{equation}
    \mathrm{P}_\mathrm{F}(\mathcal{T}_\mathrm{cal}) \leq \alpha.
\end{equation}

\subsection{Additional Experimental Results}

All experiments presented in this paper were run on a single NVIDIA GTX-1080Ti GPU with PyTorch. 


In addition, we provide the detection probabilities for CIFAR100 as the in-distribution dataset in \ref{table:cifar100}. We consider the Mahalanobis scores and Gram scores from the individual layers for the same. Recall that the energy score is a temperature scaled log-sum-exponent of the softmax scores, i.e., $-T \log \sum_{i=1}^c e^{\sigma_i(x) /T} $ where $c$ is the number of classes, $\sigma_i$ are the softmax scores, and $T$ is the temperature parameter. We do not consider the energy score as one of the statistics for CIFAR100 as in-distribution, as we do not expect it to give a good representation of the in-distribution data. As the number of classes in CIFAR100 is quite large (100), we expect the softmax scores to not provide a reliable confidence score for distinguishing in-distribution points from OOD samples. Table \ref{table:auroc_cifar100} contains the AUROC numbers for CIFAR100 as the in-distribution dataset. 

\subsection{Comparison with Bonferroni inspired test}

It is possible to construct an OOD detection test adapted from the Bonferroni procedure similar to Algorithm \ref{main}, by replacing $m$ with:
\begin{equation}
    m =  \left| \left\{ i: \hat{Q}^i \leq \frac{\alpha}{(1 + \epsilon)K} \right\} \right|,
\end{equation}

i.e., calculating $m$ as the number of hypotheses $i$ for which the corresponding conformal p-value is smaller than the constant $\frac{\alpha}{(1 + \epsilon) K}$. A sample is declared as OOD if $m \geq 1$. This procedure is detailed in Algorithm \ref{alg:bonferroni} for completeness. 

\begin{algorithm}[tb]
   \caption{Bonferroni based OOD detection test with conformal p-values}
   \label{alg:bonferroni}
\begin{algorithmic}
   \STATE {\bfseries Inputs:} 
   \STATE New input $X_\mathrm{test}$; \\ Scores over $\mathcal{T}_{cal}$ as $\left\{ \{T^1_j = s^1(X_j) : j \in \mathcal{T}_\mathrm{cal}\}, \ldots, \{T^K_j = s^K(X_j) : j \in \mathcal{T}_\mathrm{cal}\} \right\}$; \\ ML model $f(\mathbf{W}, .)$; \\ Desired conditional probability of false alarm $\alpha \in (0,1)$.
   \STATE {\bfseries Algorithm:} 
   \STATE For $X_\mathrm{test}$, compute scores $T^i_\mathrm{test}$.
   \STATE Calculate conformal p-values as:
   \begin{equation}\label{eq:conformal_p_value}
       \hat{Q}^ i = \frac{1 + |\{j \in \mathcal{T}_\mathrm{cal}: T^i_j \geq T^i_\mathrm{test}\}|}{1 + |\mathcal{T}_\mathrm{cal}|}.
   \end{equation}
   \STATE Calculate $m = \left| \left\{ i: \hat{Q}^i \leq \frac{\alpha}{(1 + \epsilon) K} \right\} \right|$.
   \STATE {\bfseries{Output:}} 
   \STATE Declare OOD if $m \geq 1$.
\end{algorithmic}
\end{algorithm}

However, in general, the Bonferroni procedure has been observed to have a smaller detection power as compared to the BH procedure. We provide a comparison of the detection performance of using the Bonferroni inspired OOD detection test versus the BH inspired test proposed in Algorithm \ref{main} in Tables \ref{table:comparison_CIFAR10} and \ref{table:comparison_SVHN}.
We can provide guarantees on the conditional false alarm probability similar to Theorem \ref{main_theorem} for Algorithm \ref{alg:bonferroni} as well. 

\begin{theorem}
    Let $\alpha, \delta \in (0,1)$. Let $\mathcal{T}_\mathrm{cal}$ be a calibration set,
and let $n_\mathrm{cal}$  be such that for a given $\delta >0$, 
\begin{equation}\label{eq:n_cal_bonferroni}
    I_{(1+\epsilon)\mu} (a, b) \geq 1 - \frac{\delta}{K}, 
\end{equation}
where $
    a = \left\lfloor (n_\mathrm{cal} + 1)\frac{\alpha }{(1 + \epsilon)K} \right\rfloor$ , 
    $b = (n_\mathrm{cal} + 1) - a$,
 $\mu = \frac{a}{a+b}$, and $I_x(a,b)$ is the CDF of a Beta distribution with parameters $a, b$. 
 Then, for a new input $X_\mathrm{test}$ and a ML model $f(\mathbf{W}, .)$, the probability of incorrectly detecting $X_\mathrm{test}$ as OOD conditioned on $\mathcal{T}_\mathrm{cal}$ while using Algorithm \ref{alg:bonferroni} is bounded by $\alpha$, \textit{i.e.},
\begin{equation}
    \mathrm{P_{H_0}}\left({\text{declare OOD }  }  | \mathcal{T}_\mathrm{cal}\right) \leq \alpha, 
\end{equation}
with probability $1 - \delta$.
\end{theorem}

\begin{proof}
    We have that 
    \begin{align}
        \mathrm{P_F}(\mathcal{T}_\mathrm{cal}) &= \mathrm{P_{H_0}}(\text{reject } \mathrm{H}_0 | \mathcal{T}_\mathrm{cal}) \\
        &= \mathrm{P_{H_0}}(m \geq 1 | \mathcal{T}_\mathrm{cal}) \\
        &= \mathrm{P_{H_0}} \left( \bigcup_{i=1}^K \left\{ \hat{Q}^i \leq \frac{\alpha}{(1 + \epsilon)K}  \right\} \bigg| \mathcal{T}_\mathrm{cal} \right) \\
        &\leq \sum_{i = 1}^K \mathrm{P_{H_0}} \left(    \left\{ \hat{Q}^i \leq \frac{\alpha}{(1 + \epsilon)K} \right\}  \bigg| \mathcal{T}_\mathrm{cal} \right)
    \end{align}

    Let $r^i = \mathrm{P_{H_0}} \left(    \left\{ \hat{Q}^i \leq \frac{\alpha}{(1 + \epsilon)K} \right\}  \bigg| \mathcal{T}_\mathrm{cal} \right)$. Thus, 
    \begin{equation}
        \mathrm{P_F}(\mathcal{T}_\mathrm{cal}) = \sum_{i=1}^K r^i.
    \end{equation}
    We have from \cite{conformal, bates} that~$r^i$ follows a Beta distribution, i.e.,~$r^i \sim \mathrm{Beta} (a, b)$, where 
\begin{align}
    a &= \left\lfloor (n_\mathrm{cal} + 1)\frac{\alpha }{(1 + \epsilon)K} \right\rfloor \\
    b &= (n_\mathrm{cal} + 1) - a.
\end{align}
The mean of this distribution is $\mu = \frac{a}{a + b}$. Let $E$ denote the event 
\begin{equation}
 E = \bigcap_{i=1}^K   \left\{r^i \leq (1 + \epsilon) \frac{\alpha}{(1+ \epsilon) K} \right\}. 
\end{equation}
When $n_\mathrm{cal}$ satisfies the condition in \eqref{eq:n_cal_bonferroni}, we have that 
\begin{align}
    1 - \mathrm{P_{H_0}}(E) &= 1 - \mathrm{P_{H_0}} \left( \bigcap_{i=1}^K   \left\{r^i \leq (1 + \epsilon) \frac{\alpha}{(1+ \epsilon) K} \right\} \right) \\
    &\leq \sum_{i=1}^K \mathrm{P_{H_0}}  \left\{r^i \geq (1 + \epsilon) \frac{\alpha}{(1+ \epsilon) K} \right\} \\
    &= \sum_{i=1}^K 1 - I_{\frac{\alpha}{K}}(a,b) \\
    &\leq \sum_{i=1}^K 1 - I_{(1 + \epsilon) \mu}(a,b) 
    \leq  \sum_{i=1}^K \frac{\delta}{K} 
    \leq \delta.
\end{align}

Thus, under event $E$, i.e., with probability greater than $1 - \delta$, we have that
\begin{equation}
\begin{split}
    \mathrm{P_F}(\mathcal{T}_\mathrm{cal}) &= \sum_{i=1}^K r^i \\
     &\leq \sum_{i=1}^K  (1 + \epsilon) \frac{\alpha}{(1+ \epsilon) K} \\
     &= \alpha.
\end{split}
\end{equation}

\end{proof}

\begin{table*}
\caption {{\small Comparison with  SOTA OOD detection techniques for CIFAR100 as in-distribution dataset. Each entry is {\bf $\mathrm{P_D} (\%)$ at $\mathrm{P_F} = 10\%$} for the corresponding detection method, OOD dataset and DNN architecture.}}
\begin{center}
\begin{adjustbox}{width=0.6\columnwidth}
\begin{tabular}{clcccc}
\\ \hline \\
\multirow{2}{*}{}  {\bf OOD Dataset} & {\bf Method} & {\bf ResNet34} & {\bf DenseNet}  \\
\\ \hline \\
SVHN   &   Mahala (penultimate layer)                               &61.75  &62.21    \\
      &   Gram 	                                      &71.60  &77.87    \\
      &   Ours - Mahala (5/4)                          &64.55  & 62.81   \\
      &   Ours - Gram (5/4)                            &58.54  & 78.15    \\
      &   Ours - Mahala, Gram (all) (5/4 + 1)         &72.81  & 70.80  \\
      \hline \\
ImageNet   &   Mahala (penultimate layer)                 &35.03  &89.05   \\
      &   Gram 	                                      &82.42  &86.42   \\
      &   Ours - Mahala (5/4)                          &86.04  &90.72    \\
      &   Ours - Gram (5/4)                            &74.43  &86.85    \\
      &   Ours - Mahala, Gram (all) (5/4 + 1)                &85.64  & 90.15 \\
      \hline \\
LSUN   &   Mahala  (penultimate layer)                    &34.00  &92.17    \\
      &   Gram 	                                      &78.36  &88.93    \\
      &   Ours - Mahala (5/4)                          &86.19  &92.86  \\
      &   Ours - Gram (5/4)                            &66.62  & 89.20   \\
      &   Ours - Mahala, Gram (all) (5/4 + 1)            &84.81  & 92.66   \\
      \hline \\
iSUN   &   Mahala  (penultimate layer)                    &36.01  &88.89    \\
      &   Gram 	                                      &83.15  &84.82    \\
      &   Ours - Mahala (5/4)                          &99.35  &99.82  \\
      &   Ours - Gram (5/4)                            &53.71  & 83.01   \\
      &   Ours - Mahala, Gram (all) (5/4 + 1)            &99.42  & 99.85 \\
      \hline 
\end{tabular}
\end{adjustbox}
\end{center}
\label{table:cifar100}
\end{table*}

\begin{table*}
\caption {{\small Comparison with  SOTA OOD detection techniques for CIFAR100 as in-distribution dataset. Each entry is {AUROC} for the corresponding detection method, OOD dataset and DNN architecture.}}
\begin{center}
\begin{adjustbox}{width=0.6\columnwidth}
\begin{tabular}{clcccc}
\\ \hline \\
\multirow{2}{*}{}  {\bf OOD Dataset} & {\bf Method} & {\bf ResNet34} & {\bf DenseNet}  \\
\\ \hline \\
SVHN   &   Mahala (penultimate layer)                    &89.35 & 85.81    \\
      &   Gram 	                                      &91.85  & 91.33   \\
      &   Ours - Mahala (5/4)                          &89.42  & 86.59   \\
      &   Ours - Gram (5/4)                            &88.86  & 91.23    \\
      &   Ours - Mahala, Gram (all) (5/4 + 1)         &91.53  & 89.98 \\
      \hline \\
ImageNet   &   Mahala (penultimate layer)                 &78.81  & 95.38   \\
      &   Gram 	                                      &94.10  & 94.13  \\
      &   Ours - Mahala (5/4)                          &94.96  & 95.65    \\
      &   Ours - Gram (5/4)                            &92.00  & 94.04   \\
      &   Ours - Mahala, Gram (all) (5/4 + 1)                &94.96  & 95.66 \\
      \hline \\
LSUN   &   Mahala  (penultimate layer)                    &78.90  & 96.39    \\
      &   Gram 	                                      &93.06  & 95.33    \\
      &   Ours - Mahala (5/4)                          &94.91  & 96.13 \\
      &   Ours - Gram (5/4)                            &90.00  & 95.19   \\
      &   Ours - Mahala, Gram (all) (5/4 + 1)            &94.73  & 96.18   \\
      \hline \\
iSUN   &   Mahala  (penultimate layer)                    &81.38  & 95.43    \\
      &   Gram 	                                      &94.71  & 94.36    \\
      &   Ours - Mahala (5/4)                          &98.12  & 98.04 \\
      &   Ours - Gram (5/4)                            &89.77  & 93.85   \\
      &   Ours - Mahala, Gram (all) (5/4 + 1)            &98.04  & 97.90 \\
      \hline \\
\end{tabular}
\end{adjustbox}
\end{center}
\label{table:auroc_cifar100}
\end{table*}

\begin{table*}
\caption {{\small Comparison with  Bonferroni inspired OOD test for CIFAR10 as in-distribution dataset. Each entry is {$\mathrm{P_D} (\%)$} at $\mathrm{P_F} = 10\%$.}}
\begin{center}
\begin{adjustbox}{width=0.6\columnwidth}
\begin{tabular}{clcccc}
\\ \hline \\
\multirow{2}{*}{}  {\bf OOD Dataset} & {\bf Method} & {\bf ResNet34} & {\bf DenseNet}  \\
\\ \hline \\
 SVHN & Mahala, Gram and Energy (BH) &97.13 &94.57 \\
      & Mahala, Gram and Energy (Bonferroni) &96.41 &91.13 \\
    \hline \\
 ImageNet & Mahala, Gram and Energy (BH) &97.03 &97.20 \\
      & Mahala, Gram and Energy (Bonferroni) &95.92 &95.89 \\
    \hline \\
 LSUN & Mahala, Gram and Energy (BH) &98.00 &97.78 \\
      & Mahala, Gram and Energy (Bonferroni) &96.99 &96.53 \\
    \hline \\
 iSUN & Mahala, Gram and Energy (BH) &97.67 &96.34 \\
      & Mahala, Gram and Energy (Bonferroni) &96.76 &94.79 \\
      \hline \\
\end{tabular}
\end{adjustbox}
\end{center}
\label{table:comparison_CIFAR10}
\end{table*}

\begin{table*}
\caption {Comparison with  Bonferroni inspired OOD test for SVHN as in-distribution dataset. Each entry is {$\mathrm{P_D} (\%)$} at $\mathrm{P_F} = 10\%$.}
\begin{center}
\begin{adjustbox}{width=0.6\columnwidth}
\begin{tabular}{clcccc}
\\ \hline \\
\multirow{2}{*}{}  {\bf OOD Dataset} & {\bf Method} & {\bf ResNet34} & {\bf DenseNet}  \\
\\ \hline \\
 CIFAR10 & Mahala, Gram and Energy (BH) &97.35 &95.23 \\
      & Mahala, Gram and Energy (Bonferroni) &95.84 &91.77 \\
    \hline \\
 ImageNet & Mahala, Gram and Energy (BH) &99.84 &99.89 \\
      & Mahala, Gram and Energy (Bonferroni) &99.72 &99.79 \\
    \hline \\
 LSUN & Mahala, Gram and Energy (BH) &99.95 &100.00 \\
      & Mahala, Gram and Energy (Bonferroni) &99.89 &99.97 \\
    \hline \\
 iSUN & Mahala, Gram and Energy (BH) &99.88 &99.98 \\
      & Mahala, Gram and Energy (Bonferroni) &99.88 &99.98 \\
      \hline \\
\end{tabular}
\end{adjustbox}
\end{center}
\label{table:comparison_SVHN}
\end{table*}

\end{document}